\documentclass{article}

\usepackage{microtype}
\usepackage{graphicx}
\usepackage{subfigure}
\usepackage{booktabs} 

\usepackage[utf8]{inputenc} 
\usepackage[T1]{fontenc}    
\usepackage{hyperref}       
\usepackage{url}            
\usepackage{amsfonts}       
\usepackage{nicefrac}       
\usepackage{microtype}      
\usepackage{amssymb}
\usepackage{amsmath}
\usepackage{graphicx}
\usepackage{amsthm}
\usepackage{float}
\usepackage{color}
\usepackage{colordvi}
\usepackage{mathtools}

\usepackage{natbib}
\newtheorem{prop}{Proposition}

\newtheorem{definition}{Definition}

\newtheorem{lemma}{Lemma}
\newtheorem{corollary}{Corollary}

\newtheorem{prop2}{Proposition}

\newtheorem{lemma2}{Lemma}

\newtheorem{innercustomlemma}{Lemma}
\newenvironment{customlemma}[1]
  {\renewcommand\theinnercustomlemma{#1}\innercustomlemma}
  {\endinnercustomlemma}

\newtheorem{innercustomcor}{Corollary}
\newenvironment{customcor}[1]
  {\renewcommand\theinnercustomcor{#1}\innercustomcor}
  {\endinnercustomcor}

\newcommand{\EOC}{ \textrm{EOC}}

\DeclareMathOperator{\Tr}{Tr}

\usepackage{hyperref}

\usepackage[accepted]{icml2019}

\begin{document}

\twocolumn[
\icmltitle{On the Impact of the Activation Function on Deep Neural Networks Training}




\begin{icmlauthorlist}
\icmlauthor{Soufiane Hayou}{ox}
\icmlauthor{Arnaud Doucet}{ox}
\icmlauthor{Judith Rousseau}{ox}
\end{icmlauthorlist}

\icmlaffiliation{ox}{Department of Statistics, University of Oxford, Oxford, United Kingdom}

\icmlcorrespondingauthor{Soufiane Hayou}{soufiane.hayou@stats.ox.ac.uk}

\icmlkeywords{Machine Learning, ICML}

\vskip 0.3in
]



\printAffiliationsAndNotice{}  
\begin{abstract}
The weight initialization and the activation function of deep neural networks have a crucial impact on the performance of the training procedure. An inappropriate selection can lead to the loss of information of the input during forward propagation and the exponential vanishing/exploding of gradients during back-propagation. Understanding the theoretical properties of untrained random networks is key to identifying which deep networks may be trained successfully as recently demonstrated by \cite{samuel} who showed that for deep feedforward neural networks only a specific choice of hyperparameters known as the `Edge of Chaos' can lead to good performance. While the work by \cite{samuel} discuss trainability issues, we focus here on training acceleration and overall performance. We give a comprehensive theoretical analysis of the Edge of Chaos and show that we can indeed tune the initialization parameters and the activation function in order to accelerate the training and improve performance.
\end{abstract}

\section{Introduction}
\label{section:intro}
Deep neural networks have become extremely popular as they achieve state-of-the-art performance on a variety of important applications including language processing and computer vision; see, e.g., \cite{Goodfellow-et-al-2016}. The success of these models has motivated the use of increasingly deep networks and stimulated a large body of work to understand their theoretical properties. It is impossible to provide here a comprehensive summary of the large number of contributions within this field. To cite a few results relevant to our contributions, \cite{montufar} have shown that neural networks have exponential expressive power with respect to the depth while \cite{poole} obtained similar results using a topological measure of expressiveness.

Since the training of deep neural networks is a non-convex optimization problem, the weight initialization and the activation function will essentially determine the functional subspace that the optimization algorithm will explore. We follow here the approach of \cite{poole} and \cite{samuel} by investigating the behaviour of random networks in the infinite-width and finite-variance i.i.d. weights context where they can be approximated by a Gaussian process as established by \cite{neal}, \cite{matthews} and \cite{lee}.

In this paper, our contribution is three-fold. Firstly, we provide a comprehensive analysis of the so-called Edge of Chaos (EOC) curve and show that initializing a network on this curve leads to a deeper propagation of the information through the network and accelerates the training. In particular, we show that a feedforward ReLU network initialized on the EOC acts as a simple residual ReLU network in terms of information propagation. Secondly, we introduce a class of smooth activation functions which allow for deeper signal propagation (Proposition \ref{prop:rate_smooth_functions}) than ReLU. In particular, this analysis sheds light on why smooth versions of ReLU (such as SiLU or ELU) perform better experimentally for deep neural networks; see, e.g., \cite{clevert}, \cite{pedamonti},  \cite{ramachandran} and \cite{milletari}.
Lastly, we show the existence of optimal points on the EOC curve and we provide guidelines for the choice of such point and we demonstrate numerically the consistence of this approach. We also complement previous empirical results by illustrating the benefits of an initialization on the EOC in this context. All proofs are given in the Supplementary Material.

\section{On Gaussian process approximations of neural networks and their stability} \label{sec:Gaussian}
\subsection{ Setup and notations}

We use similar notations to those of \cite{poole} and \cite{lee}. Consider a fully connected feedforward random neural network of depth $L$, widths $(N_l)_{1\leq l \leq L}$, weights $W^l_{ij} \stackrel{iid}\sim \mathcal{N}(0, \frac{\sigma_w^2}{N_{l-1}})$ and bias $B^l_i \stackrel{iid}\sim \mathcal{N}(0,\sigma^2_b)$, where $\mathcal{N}(\mu, \sigma^{2})$ denotes the normal distribution of mean $\mu$ and variance $\sigma^{2}$. For some input $a \in \mathbb{R}^{d}$, the propagation of this input through the network is given for an activation function  $\phi:\mathbb{R} \rightarrow \mathbb{R}$ by
\begin{align}
y^1_i(a) &= \sum_{j=1}^{d} W^1_{ij} a_j + B^1_i, \\
y^l_i(a) &= \sum_{j=1}^{N_{l-1}} W^l_{ij} \phi(y^{l-1}_j(a)) + B^l_i, \quad \mbox{for } l\geq 2.
\end{align}
Throughout this paper we assume that for all $l $ the processes $y_i^{l} (.) $ are independent (across $i$) centred Gaussian processes with covariance kernels $\kappa^l$ and write accordingly $y_i^{l} \stackrel{ind}{\sim} \mathcal{GP}(0, \kappa^{l})$. This is an idealized version of the true processes corresponding to choosing  $N_{l-1}= +\infty$ (which implies, using Central Limit Theorem, that $y_i^{l} (a)$ is a Gaussian variable for any input $a$). The approximation of  $y_i^l(.)$ by a Gaussian process was first proposed by \cite{neal} in the single layer case and has been recently extended to the multiple layer case by \cite{lee} and \cite{matthews}. We recall here the expressions of the limiting Gaussian process kernels.
For any input $a \in \mathbb R^d$, $\mathbb E[y^l_i(a)] = 0$ so that for any inputs $a,b\in \mathbb R^d$
\begin{align*}
\kappa^l(a,b) &= \mathbb{E}[y^l_i(a)y^l_i(b)]\\
&= \sigma^2_b + \sigma^2_w \mathbb{E}[\phi(y^{l-1}_i(a))\phi(y^{l-1}_i(b))]\\
&= \sigma^2_b + \sigma^2_w F_{\phi} (\kappa^{l-1}(a,a), \kappa^{l-1}(a,b), \kappa^{l-1}(b,b))
\end{align*}
where $F_{\phi}$ is a function that only depends on $\phi$. This gives a recursion to calculate the kernel $\kappa^l$; see, e.g., \cite{lee} for more details. We can also express the kernel $\kappa^l(a,b)$ (which we denote hereafter by $q^l_{ab}$) in terms of the correlation $c^l_{ab}$ in the $l^{th}$ layer
\begin{align*}
q^l_{ab} = \sigma^2_b + \sigma^2_w \mathbb{E}[\phi(\sqrt{q^{l-1}_{a}} Z_1)\phi(\sqrt{q^{l-1}_{b}} U_2(c^{l-1}_{ab}))]
\end{align*}
where $q^{l-1}_{a}:=q^{l-1}_{aa}$, resp. $c^{l-1}_{ab}:=q^{l-1}_{ab}/{\sqrt{q^{l-1}_{a} q^{l-1}_{b}}}$, is the variance, resp. correlation, in the $(l-1)^{th}$ layer and $U_2(x) = x Z_1 +
\sqrt{1 - x^2} Z_2$ where $Z_1$, $Z_2$ are independent standard Gaussian random variables. 
When it propagates through the network. $q_a^l$ is updated through the layers by the recursive formula $q^l_{a} = F(q^{l-1}_a)$, where $F$ is the `variance function' given by
\begin{equation}\label{function:var}
F(x) = \sigma^2_b + \sigma^2_w \mathbb{E}[\phi(\sqrt{x} Z)^2], \quad  \quad Z\sim \mathcal{N}(0,1)
\end{equation}
Throughout this paper, $Z, Z_1, Z_2$ will always denote independent standard Gaussian variables, and $a, b$ two inputs for the network.\\
Before starting our analysis, we define the transform $V$ for a function $\phi$ defined on $\mathbb{R}$ by $V[\phi](x) = \sigma^2_w \mathbb{E}[\phi(\sqrt{x}Z)^2]$ for $x\geq0$. We have $F = \sigma_b^2 + V[\phi]$.\\

Let $E$ and $G$ be two subsets of $\mathbb{R}$. We define the following sets of functions for $k\in\mathbb{N}$ by
\begin{align*}
\mathcal{D}^k(E,G) &= \{ f:E \rightarrow G \text{ such that $f^{(k)}$ exists}\}\\
 \mathcal{C}^k(E,G) &= \{ f \in \mathcal{D}^k(E,G) \text{ such that $f^{(k)}$ is continuous}\} \\
\mathcal{D}_g^k(E,G) &= \{f \in \mathcal{D}^k(E,G) \hspace{0.05cm}:\forall j\leq k, \mathbb{E}[f^{(j)}(Z)^2] < \infty\}\\
\mathcal{C}_{g}^k(E,G) &= \{f \in \mathcal{C}^k(E,G) \hspace{0.05cm}:\forall j\leq k, \mathbb{E}[f^{(j)}(Z)^2] < \infty\}\\
\end{align*}
where $f^{(k)}$ is the $k^{\text{th}}$ derivative of $f$. When $E$ and $G$ are not explicitly mentioned, we assume $E=G=\mathbb{R}$.

\subsection{Limiting behaviour of the variance and covariance operators}\label{sec:limit}

We analyze here the limiting behaviour of $q_a^l$ and $c_{a,b}^l$ as $l$ goes to infinity. From now onwards, we will also assume without loss of generality that $c^1_{ab} \geq 0$ (similar results can be obtained straightforwardly when $c^1_{ab} \leq 0$). We first need to define the \emph{Domains of Convergence} associated with an activation function $\phi$.

\begin{definition}\label{def:domain}
Let $\phi \in \mathcal{D}^0_g$, $(\sigma_b, \sigma_w) \in( \mathbb{R}^+)^2$.\\
    (i) Domain of convergence for the variance $D_{\phi, var}$ : $(\sigma_b, \sigma_w) \in D_{\phi, var}$ if there exists $K>0$, $q\geq 0$ such that for any input $a$ with $q^1_a \leq K$, $\lim_{l \rightarrow \infty} q^l_a = q$. We denote by $K_{\phi, var}(\sigma_b, \sigma_w)$ the maximal $K$ satisfying this condition.\\
    (ii) Domain of convergence for the correlation $D_{\phi, corr}$: $(\sigma_b, \sigma_w) \in D_{\phi, corr}$ if there exists $K>0$ such that for any two inputs $a,b$ with $q^1_a,q^1_b \leq K$, $\lim_{l \rightarrow \infty} c^l_{ab} = 1$. We denote by $K_{\phi, corr}(\sigma_b, \sigma_w)$ the maximal $K$ satisfying this condition.
\end{definition}

{\it Remark}: Typically, $q$ in Definition \ref{def:domain} is a fixed point of the variance function defined in \eqref{function:var}. Therefore, it is easy to see that for any $(\sigma_b, \sigma_w)$ such that $F$ is non-decreasing and admits at least one fixed point, we have $K_{\phi, var}(\sigma_b, \sigma_w)\geq q$ where $q$ is the minimal fixed point; i.e. $q:= \min \{ x: F(x)=x \}$. Thus, if we re-scale the input data to have $q^1_a \leq q$, the variance $q^l_a$ converges to $q$. We can also re-scale the variance $\sigma_w$ of the first layer (only) to assume that $q^1_a \leq q$ for all inputs $a$.

\begin{figure*}
    \centering
    \subfigure[ReLU with $(\sigma_b, \sigma_w)=(1, 1)$]{%
    \label{fig:constant_relu}%
    \includegraphics[width=1.7in]{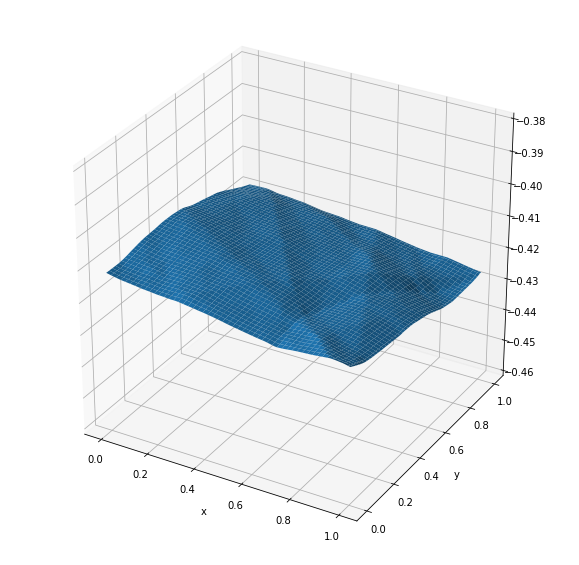}%
    }%
    \subfigure[Tanh with $(\sigma_b, \sigma_w)=(1, 1)$]{%
    \label{fig:constant_tanh}%
    \includegraphics[width=1.7in]{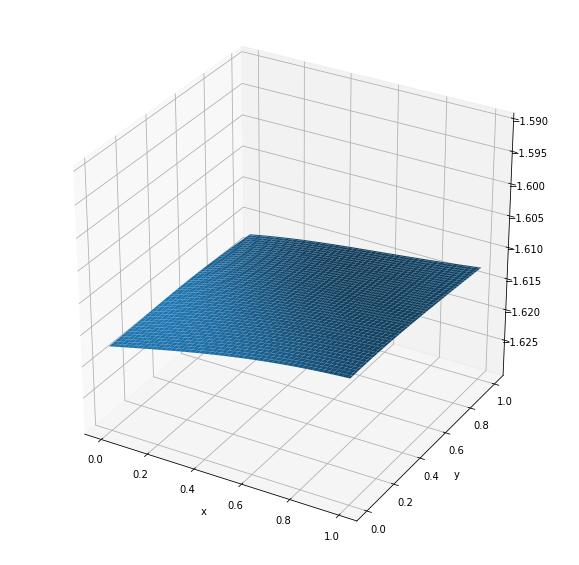}%
    }%
    \subfigure[Tanh with $(\sigma_b, \sigma_w)=(0.3, 2)$ ]{%
    \label{fig:chaotic_tanh}%
    \includegraphics[width=1.7in]{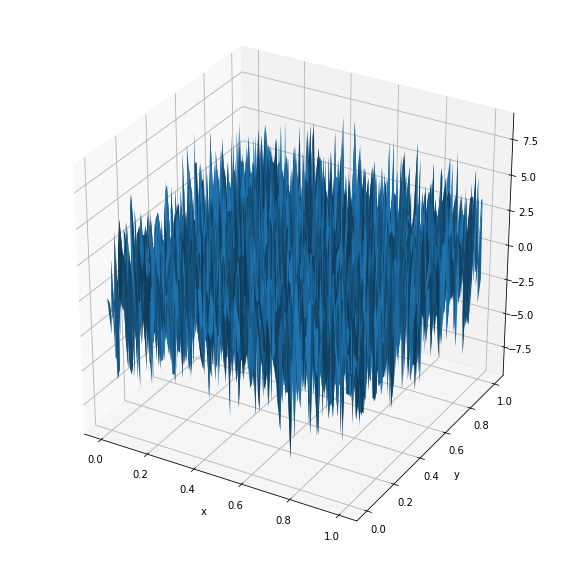}%
    }
    \caption{Draws of outputs for ReLU and Tanh networks for different parameters $(\sigma_b, \sigma_w)$. Figures (a) and (b) show the effect of an initialization in the ordered phase, the outputs are nearly constant. Figure (c) shows the effect of an initialization in the chaotic phase.}
    \label{fig:different_outputs}
\end{figure*}

The next Lemma gives sufficient conditions under which $K_{\phi, var}$ and $K_{\phi, corr}$ are infinite.

\begin{lemma}\label{lemma:infinite_K}
Assume $\phi''$ exists at least in the distribution sense.\footnote{ReLU admits a Dirac mass in 0 as second derivative and so is covered by our developments.} \\
Let $M_{\phi}:=\mathrm{sup}_{x\geq 0} \mathbb{E}[|\phi'^2(x Z) + \phi''(x Z) \phi(x Z)|] $. Assume $M_{\phi} < \infty$, then for $\sigma_w^2 < \frac{1}{M_{\phi}}$ and $\sigma_b\geq0$, we have $(\sigma_b, \sigma_w) \in D_{\phi, var}$ and $K_{\phi, var}(\sigma_b, \sigma_w) = \infty$.\\
Let  $C_{\phi, \delta} := \mathrm{sup}_{x,y \geq0, |x-y|\leq \delta, c \in [0,1]} \mathbb{E}[|\phi'(x Z_1)\phi'(y (cZ_1 + \sqrt{1-c^2}Z_2)|]$. Assume $C_{\phi, \delta}<\infty$ for some $\delta>0$, then for $\sigma^2_w < \min(\frac{1}{M_{\phi}}, \frac{1}{C_{\phi}})$ and $\sigma_b\geq0$, we have $(\sigma_b, \sigma_w) \in D_{\phi, var} \cap D_{\phi, corr}$ and $K_{\phi, var}(\sigma_b, \sigma_w)=K_{\phi, corr}(\sigma_b, \sigma_w)=\infty$.
\end{lemma}
The proof of Lemma \ref{lemma:infinite_K} is straightforward. We prove that $\sup F'(x) = \sigma_w^2 M_{\phi}$ and then apply the Banach fixed point theorem. Similar ideas are used for $C_{\phi, \delta}$.

{\it Example}: For ReLU activation function, we have $M_{ReLU} = 1/2$ and $C_{ReLU, \delta} \leq 1$ for any $\delta>0$.

In the domain of convergence $D_{\phi, var} \cap D_{\phi, corr}$,  for all $a,b\in \mathbb R^d$, we have $y_i^\infty(a ) = y_i^\infty(b)$ almost surely and the outputs of the network are constant functions. Figures \ref{fig:constant_relu} and \ref{fig:constant_tanh} illustrate this behaviour for ReLU and Tanh with inputs in $[0,1]^2$ using a network of depth $L=20$ with $N_l=300$ neurons per layer. The draws of outputs of these networks are indeed almost constant.

Under the conditions of Lemma \ref{lemma:infinite_K}, both the variance and the correlations converge exponentially fast (contraction mapping). To refine this convergence analysis, \cite{samuel} established the existence of $\epsilon_q$ and $\epsilon_c$ such that $ |q_a^l - q|\sim e^{-l/\epsilon_q}$ and $|c^l_{ab} - 1| \sim e^{-l/\epsilon_c}$ when fixed points exist. The quantities $\epsilon_q$ and $\epsilon_c$ are called `depth scales' since they represent the range of depth to which the variance and correlation can propagate without being exponentially close to their limits. More precisely, if we write $\chi_1 = \sigma^2_w \mathbb{E}[\phi'(\sqrt{q}Z)^2]$ and $\alpha  = \chi_1 + \sigma^2_w \mathbb{E}[\phi''(\sqrt{q}Z)\phi(\sqrt{q}Z)]$ then the depth scales are given by
$\epsilon_{q} = -\log(\alpha)^{-1}$ and $\epsilon_{c} = -\log(\chi_1)^{-1}$.
The equation $\chi_1=1$ corresponds to an infinite depth scale of the correlation. It is called the EOC as it separates two phases: an ordered phase where the correlation converges to 1 if $\chi_1<1$ and a chaotic phase where $\chi_1>1$  and the correlations do not converge to 1. In this chaotic regime, it has been observed in \cite{samuel} that the correlations converge to some value $c<1$ when $\phi(x)=\text{Tanh}(x)$ and that $c$ is independent of the correlation between the inputs. This means that very close inputs (in terms of correlation) lead to very different outputs. Therefore, in the chaotic phase, at the limit of infinite width and depth, the output function of the neural network is non-continuous everywhere. Figure \ref{fig:chaotic_tanh} shows an example of such behaviour for Tanh.

\begin{definition}[Edge of Chaos]\label{def:eoc}
For $(\sigma_b,\sigma_w) \in D_{\phi, var}$, let $q$ be the limiting variance\footnote{The limiting variance is a function of $(\sigma_b,\sigma_w)$ but we do not emphasize it notationally.}. The Edge of Chaos (EOC) is the set of values of $(\sigma_b,\sigma_w)$ satisfying $\chi_1=\sigma_w^2 \mathbb{E}[\phi'(\sqrt{q}Z)^2] = 1$.
\end{definition}
To further study the EOC regime, the next lemma introduces a function $f$ called the `correlation function' showing that that the correlations have the same asymptotic behaviour as the  time-homogeneous dynamical system $c^{l+1}_{ab}=f(c^l_{ab})$.
\begin{lemma}\label{correlationfunction}
Let $(\sigma_b, \sigma_w) \in D_{\phi, var} \cap D_{\phi, corr}$ such that $q>0$, $a,b \in \mathbb R^d$ and $\phi$ a measurable function such that $\sup_{x \in S} \mathbb{E}[\phi(x Z)^2]<\infty$ for all compact sets $S$. Define $f_l$ by $c_{ab}^{l+1} = f_l(c_{ab}^l)$ and $f$ by $f(x) = \frac{\sigma^2_b + \sigma^2_w \mathbb{E}[\phi(\sqrt{q}Z_1)\phi(\sqrt{q}(x Z_1 + \sqrt{1-x^2}Z_2))}{q}$. Then $\lim_{l \rightarrow \infty} \sup_{x \in [0, 1]} |f_l(x) - f(x)| = 0$.
\end{lemma}

\begin{figure*}
    \centering
    \subfigure[Convergence of the correlation to 1 with $c^0 = 0.1$]{%
    \label{fig:convergence_corre_relu}%
    \includegraphics[width=1.8in]{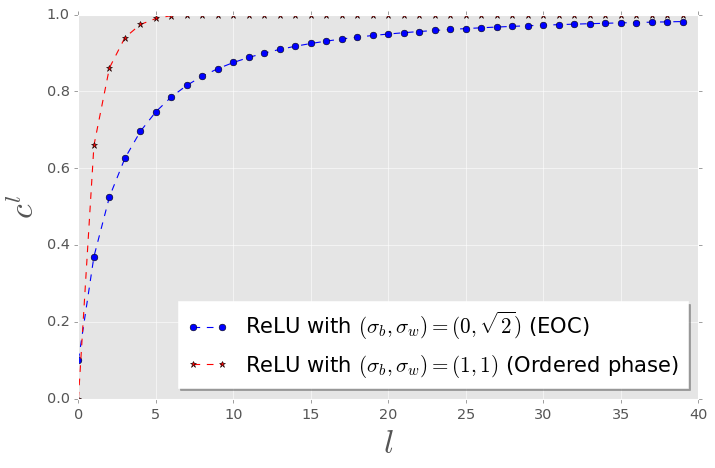}%
    }
    \subfigure[]%
    [Correlation function $f$]{%
    \label{fig:correlation_function_ReLU}%
    \includegraphics[width=1.8in]{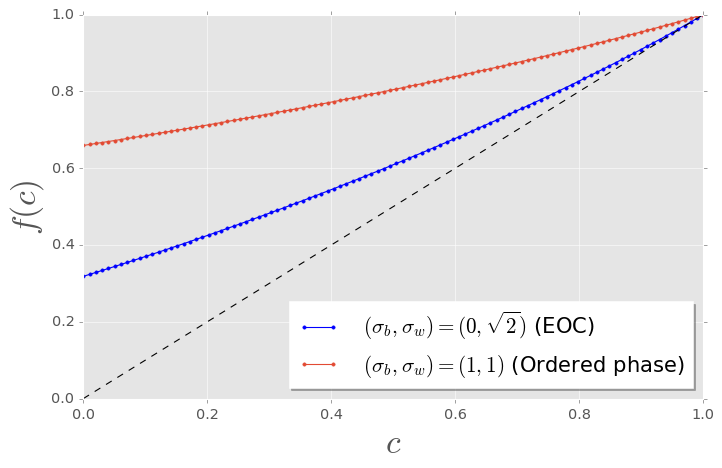}
    }%
    \subfigure[]%
    [Output of 100x20 ReLU network on the EOC]{%
    \label{fig:output_relu_100x20_edge}%
    \includegraphics[width=1.6in]{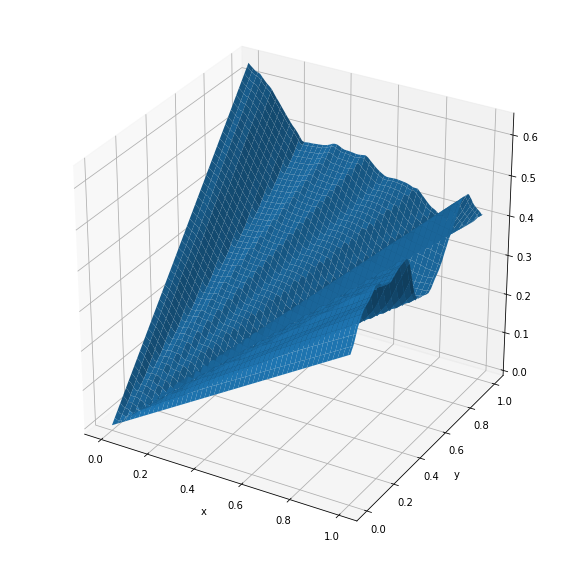}
    }
    \caption{Impact of the EOC initialization on the correlation and the correlation function. In (a), the correlation converges to 1 at a sub-exponential rate when the network is initialized on the EOC. In (b), the correlation function $f$ satisfies $f'(1)=1$ on the EOC.}
    \label{fig:constant}
\end{figure*}
The condition on $\phi$ in Lemma \ref{correlationfunction} is violated only by activation functions with square exponential growth (which are not used in practice), so from now onwards, we use this approximation in our analysis. Note that being on the EOC is equivalent to $(\sigma_b,\sigma_w)$ satisfying $f'(1)=1$. In the next section, we analyze this phase transition carefully for a large class of activation functions.

\section{Edge of Chaos}
To illustrate the effect of the initialization on the EOC, we plot in Figure \ref{fig:output_relu_100x20_edge} the output of a ReLU neural network with 20 layers and 100 neurons per layer with parameters $(\sigma_b^2, \sigma_w^2) = (0, 2)$ (as we will see later $\EOC = \{(0,\sqrt{2})\}$ for ReLU). Unlike the output in Figure \ref{fig:constant_relu}, this output displays much more variability. However, we prove below that the correlations still converge to 1 even in the EOC regime, albeit at a slower rate.

\subsection{ReLU-like activation functions} \label{sec:homoge}
ReLU has replaced classical activations (sigmoid, Tanh,...) which suffer from gradient vanishing (see e.g. \cite{glorot} and \cite{nair}). Many variants such as Leaky-ReLU were also shown to enjoy better performance in test accuracy \cite{xu}. This motivates the analysis of such functions from an initialization point of view. Let us first define this class.
\begin{definition}[ReLU-like functions]
A function $\phi$ is ReLU-like if it is of the form
\begin{equation*}
    \phi(x) =
    \begin{cases*}
      \lambda x & if $x>0$ \\
      \beta x   & if $x \leq 0$
    \end{cases*}
  \end{equation*}
where $\lambda, \beta \in \mathbb{R}$.
\end{definition}

ReLU corresponds to $\lambda = 1$ and $\beta =0$. For this class of activation functions, the EOC in terms of definition \ref{def:eoc} is reduced to the empty set.
However, we can define a weak version of the EOC for this class. From Lemma \ref{lemma:infinite_K}, when $\sigma_w< \sqrt{\frac{2}{\lambda^2 + \beta^2}}$, the variances converge to $q=\frac{\sigma_b^2}{1 - \sigma_w^2/2}$ and the correlations converge to 1 exponentially fast. If $\sigma_w>\sqrt{\frac{2}{\lambda^2 + \beta^2}}$ the variances converge to infinity. We then have the following result.
\begin{lemma}[Weak EOC]
Let $\phi$ be a ReLU-like function with $\lambda, \beta$ defined as above. Then $f'_l$ does not depend on $l$, and $f_l'(1) = 1$ and $q^l$ bounded holds if and only if $(\sigma_b, \sigma_w) = (0, \sqrt{\frac{2}{\lambda^2 + \beta^2}})$.\\
We call the singleton $\{ (0, \sqrt{\frac{2}{\lambda^2 + \beta^2}})\}$ the weak EOC.
\end{lemma}
The non existence of EOC for ReLU-like activation in the sense of definition \ref{def:eoc} is due to the fact that the variance is unchanged ($q_a^l=q_a^1$) on the weak EOC, so that the limiting variance $q$ depends on $a$. However, this does not impact the analysis of the correlations, therefore, hereafter the weak EOC is also called the EOC.\\
This class of activation functions has the interesting property of preserving the  variance across layers when the network is initialized on the EOC. We show in Proposition \ref{prop:relukernel} below that, in the EOC regime, the correlations converge to 1 at a slower rate (slower than exponential). We only present the result for ReLU but the generalization to the whole class is straightforward.\\
\begin{figure*}
    \centering
    \subfigure[Tanh]{%
    \label{fig:tanh_eoc}%
    \includegraphics[width=1.8in]{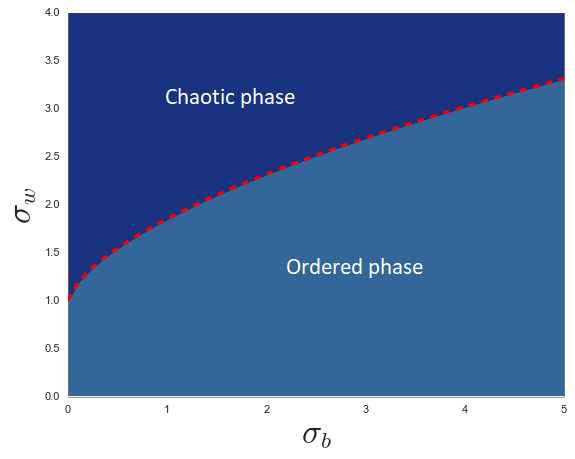}%
    }%
    \subfigure[ReLU]{%
    \label{fig:relu_eoc}%
    \includegraphics[width=1.8in]{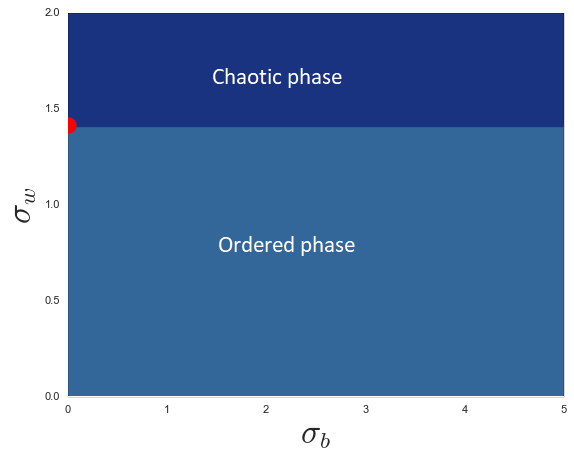}%
    }%
    \subfigure[ELU]{%
    \label{fig:elu_eoc}%
    \includegraphics[width=1.8in]{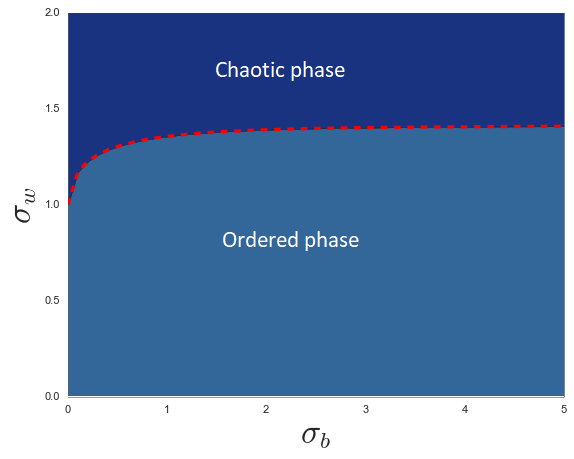}%
    }
    \caption{EOC curves for different activation functions (red dashed line). For smooth activation functions (Figures (a) and (c)), the EOC is a curve in the plane $(\sigma_b, \sigma_w)$, while it is reduced to a single point for ReLU.}
    \label{fig:eocs}
\end{figure*}

\noindent \textbf{Example: ReLU}:
The EOC is reduced to the singleton $(\sigma_b^2, \sigma_w^2) = (0,2)$, hence we should initialize ReLU networks using the parameters $(\sigma_b^2, \sigma_w^2) = (0,2)$. This result coincides with the recommendation in \cite{he2} whose objective was to make the variance constant as the input propagates but who did not analyze the propagation of the correlations. \cite{klambauer} performed a similar analysis by using the `Scaled Exponential Linear Unit' activation (SELU) that makes it possible to center the mean and normalize the variance of the post-activation $\phi(y)$. The propagation of the correlations was not discussed therein either. \\
Figure \ref{fig:correlation_function_ReLU} displays the correlation function $f$ for two different sets of parameters $(\sigma_b, \sigma_w)$. The blue graph corresponds to the EOC $(\sigma^2_b, \sigma^2_w) = (0,2)$, and the red one corresponds to an ordered phase $(\sigma_b, \sigma_w) = (1,1)$.\\
In the next result, we show that a fully connected feedforward ReLU network initialized on the EOC (weak sense) acts as if it has residual connections in terms of correlation propagation. This could potentially explain why training ReLU is faster on the EOC (see experimental results). We further show that the correlations converge to 1 at a polynomial rate of $1/{l^2}$ on the EOC instead of an exponential rate in the ordered phase.

\begin{prop}[EOC acts as Residual connections] \label{prop:relukernel}
Consider a ReLU network with parameters $(\sigma_b^2, \sigma_w^2) = (0,2) \in EOC$ and correlations $c^l_{ab}$. Consider also a ReLU network with simple residual connections given by
$$ \overline{y}^l_i(a) = \overline{y}^{l-1}_i(a) +  \sum_{j=1}^{N_{l-1}} \overline{W}^l_{ij} \phi(\overline{y}^{l-1}_j(a)) + \overline{B}^l_i $$
where $\overline{W}^l_{ij} \sim \mathcal{N}(0, \frac{\overline{\sigma}_w^2}{N_{l-1}})$ and $\overline{B}^l_i \sim \mathcal{N}(0, \overline{\sigma}_b^2)$. Let $\overline{c}^l_{ab}$ be the corresponding correlation. Then, for any $\overline{\sigma}_w>0$ and $\overline{\sigma}_b =0$, there exists a constant $\gamma>0$ such that
\begin{equation*}
    1 - c^l_{ab} \sim \gamma ( 1 - \overline{c}^l_{ab} )\sim \frac{9 \pi^2}{2 l^2} \quad \text{ as } \hspace{0.2cm} l \rightarrow \infty
\end{equation*}
\end{prop}

\subsection{Smooth activation functions} \label{sec:newclas}
We show that smooth activation functions provide better signal propagation through the network. We start by a result on the existence of the EOC.

\begin{prop}\label{prop:existence_eoc}
Let $\phi \in \mathcal{D}^1_g$ be non ReLU-like such that $\phi(0)=0$ and $\phi'(0)\neq0$. Assume that $V[\phi]$ is non-decreasing and $V[\phi']$ is non-increasing. Let $\sigma_{max}:=\sqrt{\sup_{x\geq 0}|x - \frac{V[\phi](x)}{V[\phi'](x)}|}$ and for $\sigma_b < \sigma_{max}$ let $q_{\sigma_b}$ be the smallest fixed point of the function $\sigma_b^2 + \frac{V[\phi]}{V[\phi']}$. Then we have $EOC = \{ (\sigma_b,\frac{1}{\sqrt{\mathbb{E}[\phi'(\sqrt{q_{\sigma_b}}Z)^2]}} ): \sigma_b < \sigma_{max}\}$.
\end{prop}

{\it Example :} Tanh and ELU (defined by $\phi_{ELU}(x) = x$ for $x\geq0$ and $\phi_{ELU}(x) = e^x - 1$ for $x<0$) satisfy all conditions of Proposition \ref{prop:existence_eoc}. We prove in the Appendix that SiLU (a.k.a Swish) has an EOC.\\
Using Proposition \ref{prop:existence_eoc}, we propose Algorithm \ref{alg:eoc} to determine the EOC curves.\\
Figure \ref{fig:eocs} shows the EOC curves for different activation functions. For ReLU, the EOC is reduced to a point while smooth activation functions have an EOC curve (ELU is a smooth approximation of ReLU). \\
\begin{algorithm}[tb]
   \caption{EOC curve}
   \label{alg:eoc}
\begin{algorithmic}
   \STATE {\bfseries Input:} $\phi$ satisfying conditions of Proposition \ref{prop:existence_eoc}, $\sigma_b$
   \STATE Initialize $q = 0$
   \WHILE{$q$ has not converged}
   \STATE $q = \sigma_b^2 + \frac{V[\phi](q)}{V[\phi'](q)}$
   \ENDWHILE
   \STATE {\bf return} ($\sigma_b, \frac{1}{\sqrt{V[\phi'](q)}}$)
\end{algorithmic}
\end{algorithm}
A natural question which arises from the analysis above is whether we can have $\sigma_{max} = \infty$. The answer is yes for the following large class of `Tanh-like' activation functions.

\begin{definition}[Tanh-like activation functions]
Let $\phi \in \mathcal{D}^2(\mathbb{R},\mathbb{R})$. $\phi$ is Tanh-like if
\begin{enumerate}
\item $\phi$ bounded, $\phi(0)=0$, and for all $x\in \mathbb{R}$, $\phi'(x)\geq0$, $x \phi''(x)\leq0$ and $x\phi(x)\geq0$.
\item There exist $\alpha>0$ such that $|\phi'(x)| \gtrsim e^{-\alpha |x|}$ for large $x$ (in norm).
\end{enumerate}
\end{definition}

\begin{lemma}
Let $\phi$ be a Tanh-like activation function, then $\phi$ satisfies all conditions of Proposition \ref{prop:existence_eoc} and $EOC = \{ (\sigma_b,\frac{1}{\sqrt{\mathbb{E}[\phi'(\sqrt{q}Z)^2]}} ): \sigma_b \in \mathbb{R}^+ \}$.
\end{lemma}

Recall that the convergence rate of the correlation to 1 for ReLU-like activations on the EOC is $\mathcal{O}(1/\ell^2)$. We can improve this rate by taking a sufficiently regular activation function. Let us first define a regularity class $\mathcal{A}$.

\begin{definition}
Let $\phi \in \mathcal{D}^2_{g}$. We say that $\phi$ is in $\mathcal{A}$ if there exists $n\geq 1$, a partition $(S_i)_{1\leq i \leq n}$ of $\mathbb{R}$ and $g_1, g_2, ..., g_n \in \mathcal{C}^{2}_g$ such that $\phi^{(2)} = \sum_{i=1}^n 1_{S_i} g_i$.
\end{definition}
This class includes activations such as Tanh, SiLU, ELU (with $\alpha=1$). Note that $\mathcal{D}^k_{g} \subset \mathcal{A}$ for all $k\geq3$. \\
For activation functions in $\mathcal{A}$, the next proposition shows that the correlation converges to 1 at the rate $\mathcal{O}(1/\ell)$ which is better than  $\mathcal{O}(1/\ell^2)$ of ReLU-like activation functions.

\begin{prop}[Convergence rate for smooth activations]
\label{prop:rate_smooth_functions}
Let $\phi \in \mathcal{A}$ such that $\phi$ is non-linear (i.e. $\phi^{(2)}$ is non-identically zero). Then, on the EOC, we have $1 - c^l \sim \frac{\beta_q}{l}$ where $\beta_q = \frac{2 \mathbb{E}[\phi'(\sqrt{q}Z)^2]}{q \mathbb{E}[\phi''(\sqrt{q} Z)^2]}$
\end{prop}

Choosing a smooth activation function is therefore better for deep neural networks since it provides deeper information propagation. This could explain for example why smooth versions of ReLU such as ELU perform better \cite{clevert} (see experimental results). Figure \ref{fig:convergence_corre_smoothness} shows the evolution of the correlation through the network layers for different activation functions. For function in $\mathcal{A}$ (Tanh and ELU), the graph shows a rate of $\mathcal{O}(1/\ell)$ as expected compared to $\mathcal{O}(1/\ell^2)$ for ReLU. 
\begin{figure}
    \centering
    \label{fig:convergence_corre_smoothness}%
    \includegraphics[width=2.5in]{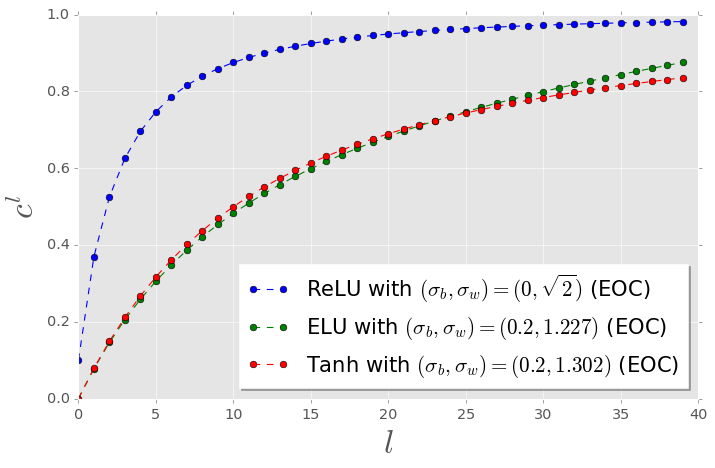}%
    \caption{Impact of the smoothness of the activation function on the convergence of the correlations on the EOC. The convergence rate for ReLU is $\mathcal{O}(1/\ell^2)$ and $\mathcal{O}(1/\ell)$ for Tanh and ELU.}
    \label{fig:conv_corr_eoc}
\end{figure}

So far, we have  discussed the impact of the EOC and the smoothness of the activation function on the behaviour of $c^l$. We now refine this analysis by studying $\beta_q$ as a function of $(\sigma_b, \sigma_w)$. We also show that $\beta_q$ plays a more important role in the information propagation process. Indeed, we show that $\beta_q$ controls the propagation of the correlation and the back-propagation of the Gradients. For the back-propagation part, we use the approximation that the weights used during forward propagation are independent of the weights used during backpropagation. This simplifies the calculations for the gradient backpropagation; see \cite{samuel} for details and \cite{yang2} for a theoretical justification.

\begin{prop}\label{prop:eoc_approximation}
Let $\phi \in \mathcal{A}$ be a non-linear activation function such that $\phi(0)=0$, $\phi'(0)\neq0$. Assume that $V[\phi]$ is non-decreasing  and $V[\phi'])$ is non-increasing, and let $\sigma_{max}>0$ be defined as in Proposition \ref{prop:existence_eoc}. Let $E$ be a differentiable loss function and define the gradient with respect to the $l^{th}$ layer by $\frac{\partial E}{\partial y^l} = (\frac{\partial E}{\partial y^l_i})_{1\leq i\leq N_l}$ and let $\Tilde{Q}^l_{ab} = \mathbb{E}[\frac{\partial E}{\partial y^l(a)} ^{T} \frac{\partial E}{\partial y^l(b)}]$ (Covariance matrix of the gradients during backpropagation). Recall that $\beta_q = \frac{2 \mathbb{E}[\phi'(\sqrt{q}Z)^2]}{q \mathbb{E}[\phi''(\sqrt{q} Z)^2]}$.

Then, for any $\sigma_b<\sigma_{max}$, by taking $(\sigma_b, \sigma_w) \in EOC$ we have
\begin{itemize}
    \item $\sup_{x \in [0,1]} |f(x) - x| \leq \frac{1}{\beta_q} $
    \item For $l \geq 1$, $\huge| \frac{\Tr(\Tilde{Q}^l_{ab})}{\Tr(\Tilde{Q}^{l+1}_{ab})} - 1\huge| \leq \frac{2}{\beta_q} $
\end{itemize}
Moreover, we have
\begin{equation*}
    \lim_{\substack{\sigma_b \rightarrow 0 \\ (\sigma_b, \sigma_w) \in \textrm{EOC}}} \beta_q = \infty.
\end{equation*}
\end{prop}

\begin{figure*}
    \centering
    \subfigure[ELU]{%
    \label{fig:training_elu}%
    \includegraphics[width=1.8in]{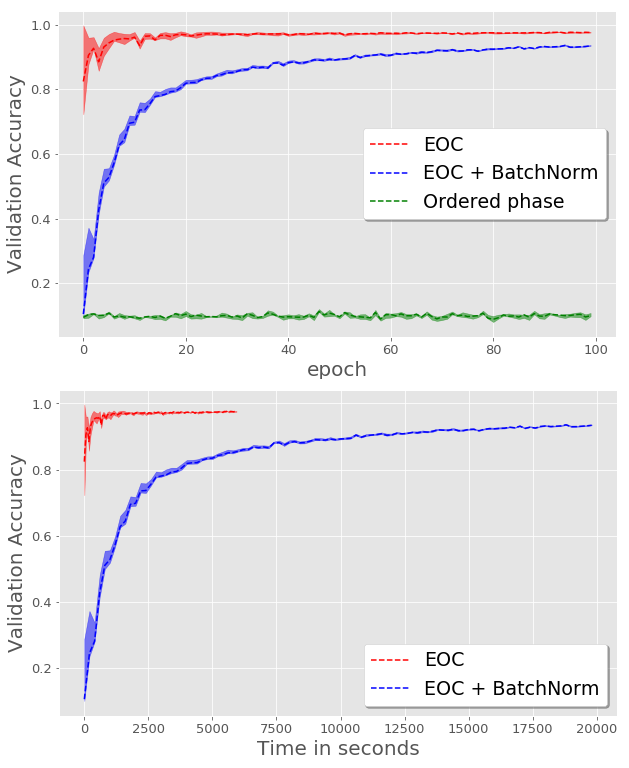}%
    }%
    \subfigure[ReLU]{%
    \label{fig:training_relu}%
    \includegraphics[width=1.8in]{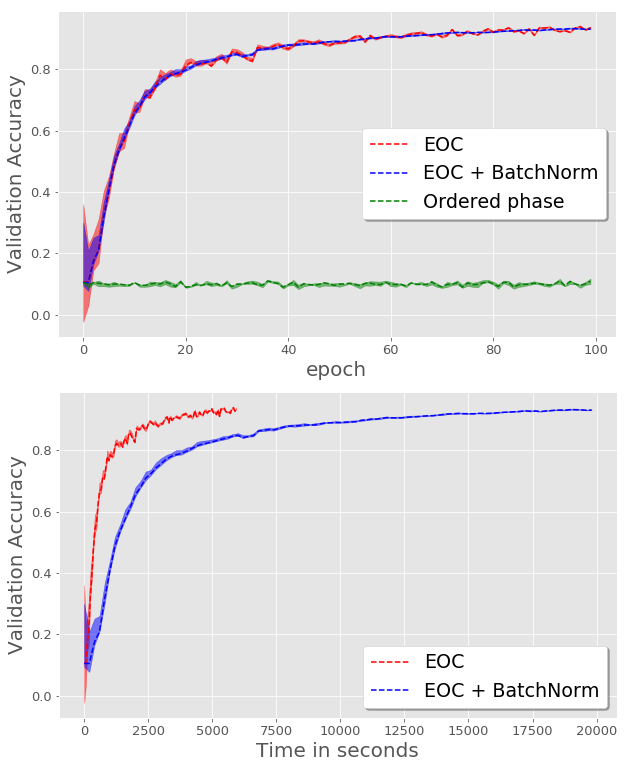}%
    }%
    \subfigure[Tanh]{%
    \label{fig:training_tanh}%
    \includegraphics[width=1.8in]{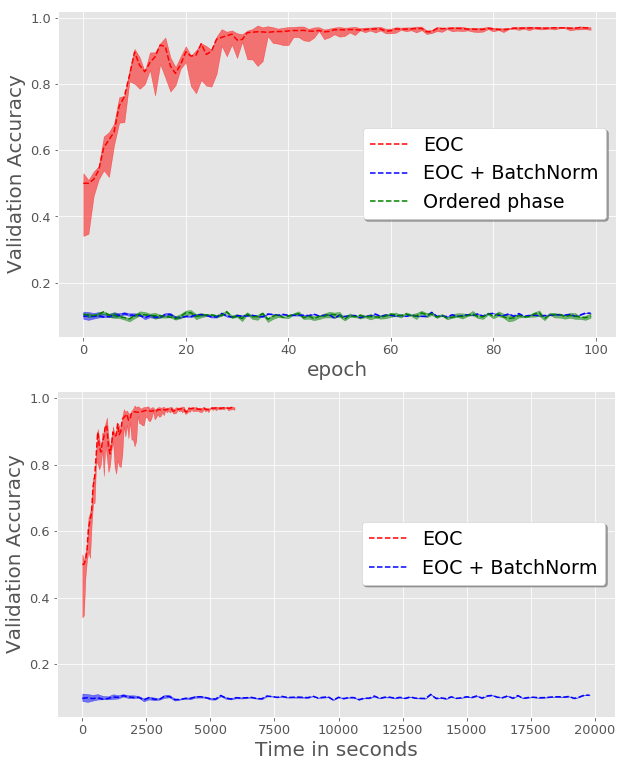}%
    }
    \caption{100 epochs of the training curve (test accuracy) for different activation functions for depth 200 and width 300 using SGD. The red curves correspond to the EOC, the green ones corresponds to an ordered phase, and the blue curves corresponds to an Initialization on the EOC plus a Batch Normalization after each layer. The upper figures show the test accuracies with respect to the epochs while the lower figures show the accuracies with respect to time.}
    \label{fig:training_300x200}
\end{figure*}

The result of Proposition \ref{prop:eoc_approximation} suggests that by taking small $\sigma_b$, we can achieve two important things. First, it makes the function $f$ close to the identity function, this slows further the convergence of the correlations to 1, i.e., the information propagates deeper inside the network. Note that the only activation functions satisfying $f(x)=x$ for all $x \in [0,1]$ are linear functions which are not useful. Second, it makes the Trace of the covariance matrix of the gradients approximately constant through layers, which means, we avoid vanishing of the information during backpropagation (More precisely, we preserve the overall spectrum of the covariance matrix since the Trace is the sum of the eigenvalues). \\
We also have $\lim_{\sigma_b \rightarrow 0} q = 0$ so that if $\sigma_b$ too small then $y^l(a) \approx 0$. Hence, a trade-off has to be taken into account when initializing on the EOC. Using Proposition \ref{prop:eoc_approximation}, we can deduce the maximal depth to which the correlations can propagate without being within a distance $\epsilon$ to 1. Indeed, we have for all $l$, $|c^{l+1} - c^l| \leq \frac{1}{\beta_q}$, therefore for $L\geq1$, $|c^L - c^0| \leq \frac{L}{\beta_q}$. Assuming $c^0 < c <1$ for all inputs where $c$ is a constant, the maximal depth we can reach without loosing $(1-\epsilon)\times 100\%$ of the information is $L_{max} = \lfloor{\beta_q (1 - c - \epsilon)}\rfloor$, this satisfies $\lim_{\sigma_b \rightarrow 0} L_{max} = \infty$.\\

{\bf Choice of $\sigma_b$ on the Edge of Chaos : }
Given a network of depth $L$, it follows that selecting a value of $\sigma_b$ on the EOC such that $\beta_q \approx L$ appears appropriate.\\

We verify numerically the benefits of this rule in the next section.\\
Note that ReLU-like activation functions do not satisfy conditions of Proposition \ref{prop:eoc_approximation}. The next lemma gives easy-to-verify sufficient conditions for Proposition \ref{prop:eoc_approximation}.

\begin{lemma}\label{lemma:easy_conditions}
Let $\phi \in \mathcal{A}$ such that $x \phi(x) \phi'(x) \geq 0$ and $\phi(x) \phi''(x) \leq 0$ for all $x \in \mathbb{R}$. Then, $\phi$ satisfies all conditions of Proposition \ref{prop:eoc_approximation}.
\end{lemma}
{\it Example}: Tanh and ELU satisfy all conditions of Lemma \ref{lemma:easy_conditions}. This may partly explain why ELU performs experimentally better than ReLU (see next section). Another example is an activation function of the form $\lambda x + \beta \text{Tanh}(x) $ where $\lambda, \beta \in \mathbb{R}$. We check the performance of these activations in the next section.

\section{Experiments}
In this section, we demonstrate empirically the theoretical results established above. We show that:
\begin{itemize}
    \item For deep networks, only an initialization on the EOC could make the training possible, and the initialization on the EOC performs better than Batch Normalization.
    \item Smooth activation functions in the sense of Proposition \ref{prop:rate_smooth_functions} perform better than ReLU-like activation, especially for very deep networks.
    \item Choosing the right point on the EOC further accelerates the training.
\end{itemize}

\begin{figure*}
    \centering
    \subfigure[Epoch 10]{%
    \label{fig:tanh_epoch10}%
    \includegraphics[width=2.125in]{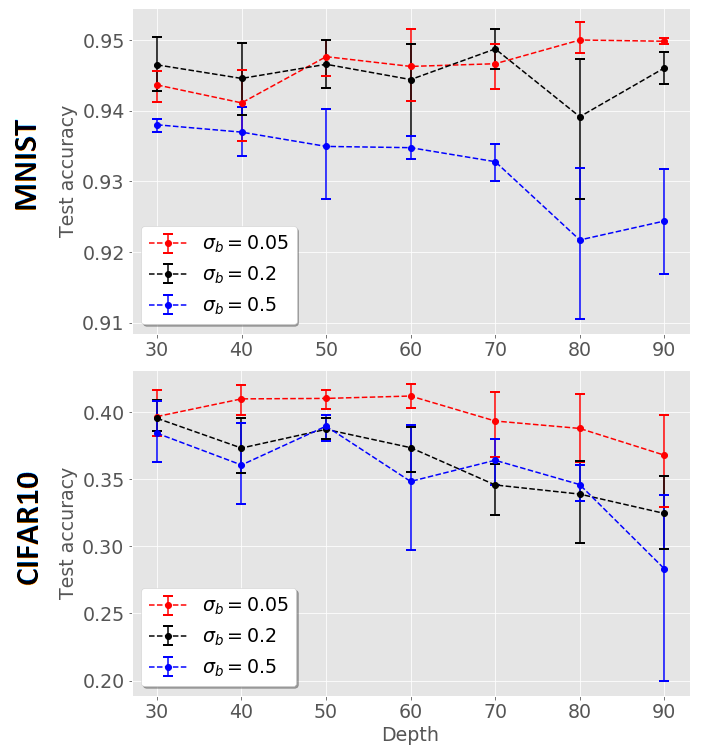}%
    }%
    \subfigure[Epoch 50]{%
    \label{fig:tanh_epoch50}%
    \includegraphics[width=2in]{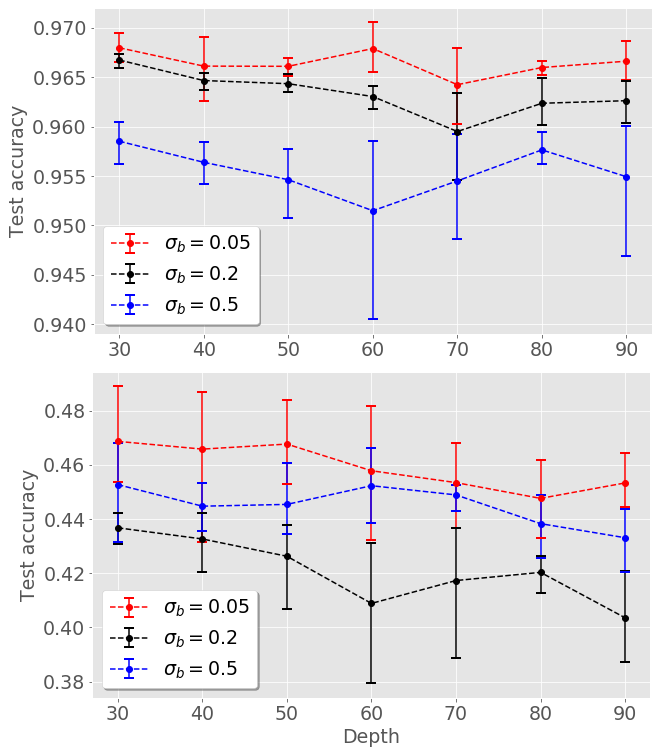}%
    }%
    \subfigure[Epoch 100]{%
    \label{fig:tanh_epoch100}%
    \includegraphics[width=2in]{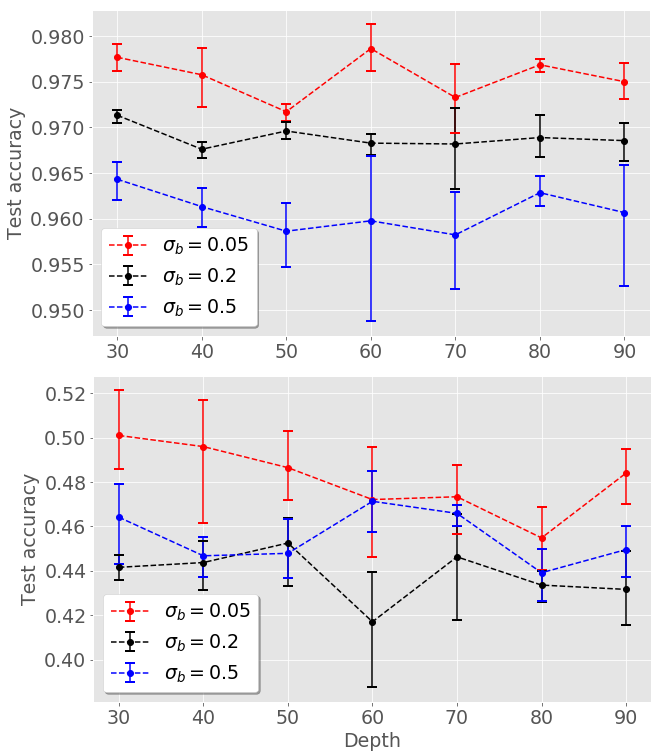}%
    }
    \caption{Test accuracies for Tanh network with depths between 30 and 90  and width 300 using different points on the EOC.}
    \label{fig:tanh_different_depths}
\end{figure*}

We demonstrate empirically our results on the MNIST and CIFAR10 datasets for depths $L$ between 10 and 200 and width $300$. We use SGD and RMSProp for training. We performed a grid search between $10^{-6}$ and $10^{-2}$ with exponential step of size 10 to find the optimal learning rate. For SGD, a learning rate of $\sim 10^{-3}$ is nearly optimal for $L \leq 150$, for $L>150$, the best learning rate is $\sim 10^{-4}$. For RMSProp,$10^{-5}$ is nearly optimal for networks with depth $L\leq200$ (for deeper networks, $10^{-6}$ gives better results). We use a batchsize of 64.\\

{\bf Initialization on the Edge of Chaos}.
We initialize randomly the network by sampling $W^l_{ij} \stackrel{iid}{\sim} \mathcal{N}(0, \sigma_w^2 /N_{l-1})$ and $B^l_i \stackrel{iid}{\sim}  \mathcal{N}(0,\sigma^2_b)$. Figure \ref{fig:training_300x200} shows that the initialization on the EOC dramatically accelerates the training for ELU, ReLU and Tanh. The initialization in the ordered phase (here we used $(\sigma_b, \sigma_w) = (1,1)$ for all activations) results in the optimization algorithm being stuck eventually at a very poor test accuracy of $\sim 0.1$ (equivalent to selecting the output uniformly at random). Figure \ref{fig:training_300x200} also  shows that EOC combined to BatchNorm results in a worse learning curve and dramatically increases the training time. Note that it is crucial here to initialize BatchNorm parameters to $\alpha=1$ and $\beta=0$ in order to keep our analysis on the forward propagation on the EOC valid for networks with BatchNorm.
\begin{table}[t]
\label{table:test_acc_100epochs}
\caption{Test accuracies for width 300 and depth 200 with different activation function on MNIST and CIFAR10 after 100 epochs}
\vskip 0.15in
\begin{center}
\begin{small}
\begin{sc}
\begin{tabular}{lcccr}
\toprule
\bf MNIST  & EOC & EOC + BN & Ord Phase \\
\midrule
ReLU    & {\bf 93.57$\pm$ 0.18}& 93.11$\pm$ 0.21 & 10.09$\pm$ 0.61  \\
ELU & {\bf 97.62$\pm$ 0.21}& 93.41$\pm$ 0.3& 10.14$\pm$ 0.51\\
Tanh    & {\bf 97.20$\pm$ 0.3}& 10.74$\pm$ 0.1& 10.02$\pm$ 0.13 \\
\bottomrule
\end{tabular}
\end{sc}
\end{small}
\end{center}

\begin{center}
\begin{small}
\begin{sc}
\begin{tabular}{lcccr}
\toprule
\bf CIFAR10 & EOC & EOC + BN & Ord Phase \\
\midrule
ReLU    & {\bf 36.55$\pm$ 1.15}& 35.91$\pm$ 1.52 & 9.91$\pm$ 0.93  \\
ELU & {\bf 45.76$\pm$ 0.91}& 44.12$\pm$ 0.93& 10.11$\pm$ 0.65\\
Tanh    & {\bf 44.11$\pm$ 1.02}& 10.15$\pm$ 0.85& 9.82$\pm$ 0.88 \\
\bottomrule
\end{tabular}
\end{sc}
\end{small}
\end{center}
\end{table}
Table \ref{table:test_acc_100epochs} presents test accuracy after 100 epochs for different activation functions and different training methods (EOC, EOC$+$BatchNorm, Ordered phase) on MNIST and CIFAR10. For all activation functions but Softplus, EOC initialization leads to the best performance. Adding BatchNorm to the EOC initialization makes the training worse, this can be explained the fact that parameters $\alpha$ and $\beta$ are also modified during the first backpropagation. This invalidates the EOC results for gradient backpropagation (see proof of Proposition \ref{prop:eoc_approximation}).

%

{\bf Impact of the smoothness of the activation function on the training.} Table \ref{table:impact_smoothness} shows the test accuracy at different epochs for ReLU, ELU, Tanh. Smooth activation functions perform better than ReLU. More experimental results with RMSProp and other activation functions of the form $x + \alpha \text{Tanh}(x)$ are provided in the supplementary material.
\begin{table}[t]
\label{table:impact_smoothness}
\caption{Test accuracies for width 300 and depth 200 with different activation function on MNIST and CIFAR10 after 10, 50 and 100 epochs}
\vskip 0.15in
\begin{center}
\begin{small}
\begin{sc}
\begin{tabular}{lcccr}
\toprule
\bf MNIST  & Epoch 10 & Epoch 50 & Epoch 100 \\
\midrule
ReLU    &  66.76$\pm$ 1.95 & 88.62$\pm$ 0.61 & 93.57$\pm$ 0.18  \\
ELU &  {\bf 96.09$\pm$ 1.55} & {\bf 97.21$\pm$ 0.31}& {\bf 97.62$\pm$ 0.21}\\
Tanh    &  89.75$\pm$ 1.01 & 96.51$\pm$ 0.51 & 97.20$\pm$ 0.3 \\
\bottomrule
\end{tabular}
\end{sc}
\end{small}
\end{center}

\begin{center}
\begin{small}
\begin{sc}
\begin{tabular}{lcccr}
\toprule
\bf CIFAR10 & Epoch 10 & Epoch 50 & Epoch 100 \\
\midrule
ReLU    &  26.46$\pm$ 1.68& 33.74$\pm$ 1.21 & 36.55$\pm$ 1.15  \\
ELU & {\bf 35.95$\pm$ 1.83}& {\bf 45.55$\pm$ 0.91}& {\bf 47.76$\pm$ 0.91}\\
Tanh    &  34.12$\pm$ 1.23& 43.47$\pm$ 1.12& 44.11$\pm$ 1.02 \\
\bottomrule
\end{tabular}
\end{sc}
\end{small}
\end{center}
\end{table}

{\bf Selection of a point on the EOC.} We have showed that a sensible choice is to select $\sigma_b$ such that $L \sim \beta_q$ on the EOC. Figure \ref{fig:tanh_different_depths} shows test accuracy of a Tanh network for different depths using $\sigma_b \in \{ 0.05, 0.2, 0.5\}$. With $\sigma_b = 0.05$, we have $\beta_q \sim 50$. We see for depth 50, the red curve ($\sigma_b = 0.05$) is the best. For other depths $L$ between 30 and 90, $\sigma_b=0.05$ is the value that makes $\beta_q$ the closest to $L$ among $ \{ 0.05, 0.2, 0.5\}$, which explains why the red curve is approximately better for all depths between 30 and 90.\\
To further confirm this finding, we search numerically for the best $\sigma_b \in \{ 2k \times 10^{-2} : k \in [1,50]\}$ for depths $30, 100, 200$. Table \ref{table:optimal_point} shows the results.

\begin{table}[htb]
\label{table:optimal_point}
\caption{Best test accuracy achieved after 100 epochs with Tanh on MNIST}
\vskip 0.15in
\begin{center}
\begin{small}
\begin{sc}
\begin{tabular}{lcccr}
\toprule
 Depth & $L=30$ &  $L=50$ & $L=200$ \\
\midrule
$2k \times 10^{-2}$    &  0.080 & 0.040 & 0.020  \\
with Rule $\beta_q \approx L$ &  0.071 & 0.030& 0.022\\
\bottomrule
\end{tabular}
\end{sc}
\end{small}
\end{center}
\end{table}

\section{Discussion}
The Gaussian process approximation of Deep Neural Networks was used by \cite{samuel} to show that very deep Tanh networks are trainable only on the EOC. We give here a comprehensive analysis of the EOC for a large class of activation functions. We also prove that smoothness plays a major role in terms of signal propagation. Numerical results in Table \ref{table:impact_smoothness} confirm this finding. Moreover, we introduce a rule to choose the optimal point on the EOC, this point is a function of the depth. As the depth goes to infinity (e.g. $L=400$), we need smaller $\sigma_b$ to achieve the best signal propagation. However, the limiting variance $q$ also becomes close to zero as $\sigma_b$ goes to zero. To avoid this problem, one possible solution is to change the activation function to ensure that the coefficient $\beta_q$ becomes large independently of the choice of $\sigma_b$ on the EOC (see supplementary material).

Our results have implications for Bayesian neural networks which have received renewed attention lately; see, e.g., \cite{hernandez} and \cite{lee}. They indeed indicate that, if one assigns i.i.d. Gaussian prior distributions to the weights and biases, we need to select not only the prior parameters $(\sigma_b,\sigma_w)$ on the EOC but also an activation function satisfying Proposition \ref{prop:rate_smooth_functions} to obtain a non-degenerate prior on the induced function space.


\bibliography{bib_icml}
\bibliographystyle{icml2019}

\appendix

\onecolumn

\section{Proofs}

We provide in this supplementary material the proofs of theoretical results presented in the main document, and we give additive theoretical and experimental results. For the sake of clarity we recall the results before giving their proofs.

\subsection{Convergence to the fixed point: Proposition 1}

\begin{lemma2}
Let $M_{\phi} := \mathrm{sup}_{x\geq 0} \mathbb{E}[|\phi'^2(x Z) + \phi''(x Z) \phi(x Z)|] $. Suppose $M_{\phi} < \infty$, then for $\sigma_w^2 < \frac{1}{M_{\phi}}$ and any $\sigma_b$, we have $(\sigma_b, \sigma_w) \in D_{\phi, var}$ and $K_{\phi, var}(\sigma_b, \sigma_w) = \infty$

Moreover, let  $C_{\phi, \delta} := \mathrm{sup}_{x,y \geq0, |x-y|\leq \delta, c \in [0,1]} \mathbb{E}[|\phi'(x Z_1)\phi'(y (cZ_1 + \sqrt{1-c^2}Z_2)|]$. Suppose $C_{\phi, \delta}<\infty$ for some positive $\delta$, then for $\sigma^2_w < \min(\frac{1}{M_{\phi}}, \frac{1}{C_{\phi}})$ and any $\sigma_b$, we have $(\sigma_b, \sigma_w) \in D_{\phi, var} \cap D_{\phi, corr}$ and $K_{\phi, var}(\sigma_b, \sigma_w)=K_{\phi, corr}(\sigma_b, \sigma_w)=\infty$.
\end{lemma2}
\medskip

\begin{proof}
To abbreviate the notation, we use $q^l := q^l_{a}$ for some fixed input $a$.\\

\noindent \textbf{Convergence of the variances:}  We first consider the asymptotic behaviour of $q^l = q_a^l$. Recall that $q^l = F(q^{l-1})$
where
\begin{equation*}
    F(x) = \sigma^2_b + \sigma^2_w \mathbb{E}[\phi(\sqrt{x} Z)^2].
\end{equation*}
The first derivative of this function is given by
\begin{align} \label{Fderiv}
F'(x) &= \sigma^2_w \mathbb{E}[\frac{Z}{\sqrt{x}}\phi'(\sqrt{x} Z)\phi(\sqrt{x} Z)]= \sigma^2_w \mathbb{E}[\phi'(\sqrt{x} Z)^2 + \phi''(\sqrt{x} Z)\phi(\sqrt{x} Z)],
\end{align}
where we use Gaussian integration by parts, $\mathbb{E}[ZG(Z)] = \mathbb{E}[G'(Z)]$, an identity satisfied by any function $G$ such that $\mathbb{E}[|G'(Z)|]<\infty$.

Using the condition on $\phi$, we see that the function $F$ is a contraction mapping for $\sigma^2_w < \frac{1}{M_{\phi}}$ and  the Banach fixed-point theorem guarantees the existence of a unique fixed point $q$ of $F$, with $\lim_{l\rightarrow +\infty } q^l=q$.  Note that this fixed point depends only on $F$, therefore this is true for any input $a$ and $K_{\phi, var}(\sigma_b, \sigma_w) = \infty$.\\

\noindent \textbf{Convergence of the covariances:}
Since $M_{\phi}<\infty$, then for all $a,b \in \mathbb R^d$ there exists $l_0$ such that $|\sqrt{q^l_{a}} - \sqrt{q^l_{b}}|<\delta$ for all $l>l_0$. Let $l>l_0$, using Gaussian integration by parts, we have
\begin{equation*}
    \frac{d c^{l+1}_{ab}}{dc^l_{ab}} = \sigma^2_w \mathbb{E}[|\phi'(\sqrt{q^l_{a}} Z_1)\phi'(\sqrt{q^l_{b}} (c^l_{ab} Z_1 + \sqrt{1-(c^l_{ab})^2}Z_2)|] .
\end{equation*}
We cannot use the Banach fixed point theorem directly because the integrated function here depends on $l$ through $q^l$. For ease of notation, we write $c^l := c^l_{ab}$. We have
\begin{align*}
|c^{l+1} - c^l| &= |\int_{c^{l-1}}^{c^l} \frac{d c^{l+1}}{dc^l}(x) dx| \leq \sigma^2_w C_{\phi} |c^{l} - c^{l-1}|.
\end{align*}
Therefore, for $\sigma^2_w < \min(\frac{1}{M_{\phi}}, \frac{1}{C_{\phi}})$, $c^l$ is a Cauchy sequence and it converges to a limit $c\in [0,1]$. At the limit
\begin{equation*}
c= f(c) = \frac{\sigma^2_b + \sigma^2_w \mathbb{E}[\phi(\sqrt{q} z_1)\phi(\sqrt{q}(c z_1 + \sqrt{1 - c^2} z_2)))]}{q}.
\end{equation*}
The derivative of this function is given by
\begin{equation*}
f'(x) =  \sigma^2_w \mathbb{E}[\phi'(\sqrt{q} Z_1)\phi'(\sqrt{q} (x Z_1 + \sqrt{1-x}Z_2)].
\end{equation*}
By assumption on $\phi$ and the choice of $\sigma_w$, we have $\textrm{sup}_{x}|f'(x)|<1$ so $f$ is a contraction and has a unique fixed point. Since $f(1) = 1$ then $c=1$.
The above result is true for any $a,b$, therefore $K_{\phi, var}(\sigma_b, \sigma_w)=K_{\phi, corr}(\sigma_b, \sigma_w)=\infty$.
\end{proof}

\medskip
\medskip

\begin{lemma2}
Let $(\sigma_b, \sigma_w) \in D_{\phi, var} \cap D_{\phi, corr}$ such that $q>0$, $a,b \in \mathbb R^d$ and $\phi$ an activation function such that $\sup_{x \in K} \mathbb{E}[\phi(x Z)^2]<\infty$ for all compact sets $K$. Define $f_l$ by $c_{a,b}^{l+1} = f_l(c_{a,b}^l)$ and $f$ by $f(x) = \frac{\sigma^2_b + \sigma^2_w \mathbb{E}[\phi(\sqrt{q}Z_1)\phi(\sqrt{q}(x Z_1 + \sqrt{1-x^2}Z_2))}{q}$. Then $\lim_{l \rightarrow \infty} \sup_{x \in [0, 1]} |f_l(x) - f(x)| = 0$.
\end{lemma2}

\begin{proof}
For $x \in [0,1]$, we have
\begin{align*}
f_l(x) - f(x) &= (\frac{1}{\sqrt{q^l_a q^l_b}} - \frac{1}{q}) (\sigma_b^2 + \sigma_w^2 \mathbb{E}[\phi(\sqrt{q^l_a}Z_1)\phi(\sqrt{q^l_b} u_2(x))]) \\
&+ \frac{\sigma_w^2}{q} (\mathbb{E}[\phi(\sqrt{q^l_a}Z_1)\phi(\sqrt{q^l_b} u_2(x))] - \mathbb{E}[\phi(\sqrt{q}Z_1)\phi(\sqrt{q} u_2(x))]),
\end{align*}
where $u_2(x) := x Z_1 + \sqrt{1-x^2} Z_2$. The first term goes to zero uniformly in $x$ using the condition on $\phi$ and Cauchy-Schwartz inequality. As for the second term, it can be written again as
\begin{equation*}
\mathbb{E}[(\phi(\sqrt{q^l_a}Z_1)-\phi(\sqrt{q}Z_1))\phi(\sqrt{q^l_b} u_2(x))] + \mathbb{E}[\phi(\sqrt{q}Z_1)(\phi(\sqrt{q^l_b} u_2(x)) - \phi(\sqrt{q} u_2(x)))].
\end{equation*}
Using Cauchy-Schwartz and the condition on $\phi$, both terms can be controlled uniformly in $x$ by an integrable upper bound. We conclude using dominated convergence.
\end{proof}

\medskip
\medskip

\begin{lemma2}[Weak EOC]
Let $\phi$ be a ReLU-like function with $\lambda, \beta$ defined as above. Then $f'_l$ does not depend on $l$, and having $f_l'(1) = 1$ and $q^l$ bounded is only achieved for the singleton $(\sigma_b, \sigma_w) = (0, \sqrt{\frac{2}{\lambda^2 + \beta^2}})$. The Weak EOC is defined as this singleton.
\end{lemma2}

\begin{proof}
We write $q^l = q_a^l$ throughout the proof.
Note first that  the variance satisfies the recursion:
\begin{align}\label{recursion}
q^{l+1} &= \sigma_b^2 + \sigma_w^2  \mathbb{E}[\phi(Z)^2] q^l = \sigma_b^2 + \sigma_w^2 \frac{\lambda^2 + \beta^2}{2} q^l.
\end{align}
For all $\sigma_w < \sqrt{\frac{2}{\lambda^2 + \beta^2}}$, $q = \sigma^2_b\left( 1 - \sigma_w^2 (\lambda^2 + \beta^2)/2 \right)^{-1}$ is a fixed point.
This is true for any input, therefore $K_{\phi, var}(\sigma_b, \sigma_w) = \infty$ and (i) is proved.

Now, the EOC equation is given by $\chi_1 = \sigma_w^2 \mathbb{E}[\phi'(Z)^2] = \sigma_w^2 \frac{\lambda^2 + \beta^2}{2}$. Therefore, $\sigma^2_w = \frac{2}{\lambda^2 + \beta^2}$. Replacing $\sigma_w^2$ by its critical value in \eqref{recursion} yields
\begin{align*}
q^{l+1} &= \sigma_b^2 +  q^l.
\end{align*}
Thus $q= \sigma_b^2 + q$ if and only if $\sigma_b= 0$, otherwise $q^l$ diverges to infinity. So the frontier is reduced to a single point  $(\sigma^2_b, \sigma^2_w) = ( 0, \mathbb{E}[\phi'(Z)^2]^{-1})$, and the variance does not depend on $l$.

\end{proof}

\begin{prop2}[EOC acts as Residual connections] \label{prop:relukernel}
Consider a ReLU network with parameters $(\sigma_b^2, \sigma_w^2) = (0,2) \in EOC$ and let $c^l_{ab}$ be the corresponding correlation. Consider also a ReLU network with simple residual connections given by
$$ \overline{y}^l_i(a) = \overline{y}^{l-1}_i(a) +  \sum_{j=1}^{N_{l-1}} \overline{W}^l_{ij} \phi(\overline{y}^{l-1}_j(a)) + \overline{B}^l_i, $$
where $\overline{W}^l_{ij} \sim \mathcal{N}(0, \frac{\overline{\sigma}_w^2}{N_{l-1}})$ and $\overline{B}^l_i \sim \mathcal{N}(0, \overline{\sigma}_b^2)$. Let $\overline{c}^l_{ab}$ be the corresponding correlation. Then, by taking $\overline{\sigma}_w >0$ and $\overline{\sigma}_b =0$, there exists a constant $\gamma>0$ such that
\begin{equation*}
    1 - c^l_{ab} \sim \gamma ( 1 - \overline{c}^l_{ab} )\sim \frac{9 \pi^2}{2 l^2}
\end{equation*}
as $l \rightarrow \infty$.
\end{prop2}

\begin{proof}

Let us first give a closed-form formula of the correlation function $f$ of a ReLU network. In this case, we have $f(x) = 2 \mathbb{E}[ (Z_1)_+ (x Z_1 + \sqrt{1-x^2}Z_2)_+] $ where $(x)_+ := x 1_{x>0}$. Let $x \in [0,1]$, $f$ is differentiable and satisfies
\begin{equation*}
    f'(x) = 2 \mathbb{E}[ 1_{Z_1>0} 1_{x Z_1 + \sqrt{1-x^2}Z_2>0}],
\end{equation*}
which is also differentiable. Simple algebra leads to
$$f^{"}(x) =  \frac{1}{\pi \sqrt{1 - x^2}}.$$
Since $\arcsin'(x) = \frac{1}{\sqrt{1-x^2}}$ 	and $f'(0) = 1/2$,
\begin{equation*}
    f'(x) = \frac{1}{\pi} \arcsin(x) + \frac{1}{2}.
\end{equation*}
Using the fact that $\int \arcsin = x \arcsin + \sqrt{1 - x^2}$ and $f(1)=1$, we conclude that for $x \in [0,1]$, $f(x) = \frac{1}{\pi} x \arcsin(x) + \frac{1}{\pi} \sqrt{1 - x^2} + \frac{1}{2} x $.\\
For the residual network, we have $\overline{q}^l_a = \overline{q}^{l-1}_a + \overline{\sigma}_w^2 \mathbb{E}[\phi(\sqrt{\overline{q}^{l-1}_a} Z)^2] = (1 + \frac{\overline{\sigma}_w^2}{2}) \overline{q}^{l-1}_a$. \\
Let $\delta = \frac{1}{1 + \frac{\overline{\sigma}_w^2}{2}}$. We have

\begin{align*}
    \overline{c}^l_{ab} &= \delta  \overline{c}^{l-1}_{ab} + \delta  \overline{\sigma}_w^2 \mathbb{E}[\phi(Z_1)\phi(U_2(\overline{c}^{l-1}_{ab}))]\\
    &=  \overline{c}^{l-1}_{ab} + \delta  \frac{\overline{\sigma}_w^2}{2} (f(\overline{c}^{l-1}_{ab}) - \overline{c}^{l-1}_{ab})\\
\end{align*}
Now, we use Taylor expansion near to conclude. However, since $f$ is not differentiable in 1 for all orders, we use a change of variable $x= 1 - t^2$ with $t $ close to 0, then
\begin{equation*}
    \arcsin(1-t^2) = \frac{\pi}{2} -\sqrt{2} t - \frac{\sqrt{2}}{12} t^3 + O(t^5),
\end{equation*}
so that
\begin{equation*}
    \arcsin(x) = \frac{\pi}{2} -\sqrt{2} (1-x)^{1/2} - \frac{\sqrt{2}}{12} (1-x)^{3/2} + O((1-x)^{5/2}),
\end{equation*}
and
\begin{equation*}
    x \arcsin(x) = \frac{\pi}{2}x -\sqrt{2} (1-x)^{1/2} + \frac{11 \sqrt{2}}{12} (1-x)^{3/2} + O((1-x)^{5/2}).
\end{equation*}
Since
\begin{equation*}
    \sqrt{1-x^2} = \sqrt{2}(1-x)^{1/2} - \frac{\sqrt{2}}{4} (1-x)^{3/2} + O((1-x)^{5/2}),
\end{equation*}
we obtain that
\begin{equation}\label{Taylorf}
f(x) \underset{x \rightarrow 1-}{=} x + \frac{2 \sqrt{2}}{3 \pi} (1-x)^{3/2} + O((1-x)^{5/2}).
\end{equation}
Since $(f(x) - x)'= \frac{1}{\pi}(\arcsin(x) - \frac{\pi}{2}) < 0$ and $f(1)=1$, for all  $x \in [0,1)$, $f(x) > x$. If $c^l < c^{l+1}$ then by taking the image by $f$ (which is increasing because $f'\geq0$) we have that $c^{l+1} < c^{l+2}$, and we know that $c^1 = f(c^0) \geq c^0$, so by induction the sequence $c^l$ is increasing, and therefore it converges to the fixed point of $f$ which is 1.\\

Using a Taylor expansion of $f$ near 1, we have 
\begin{equation*}
    \overline{c}^l_{ab} = \overline{c}^{l-1}_{ab} + \delta \frac{2\sqrt{2}}{3 \pi} (1 - \overline{c}^{l-1}_{ab})^{3/2} + O((1 - \overline{c}^{l-1}_{ab})^{5/2})
\end{equation*}
and
\begin{equation*}
    c^l_{ab} = c^{l-1}_{ab} +  \frac{2\sqrt{2}}{3 \pi} (1 - c^{l-1}_{ab})^{3/2} + O((1 - c^{l-1}_{ab})^{5/2}).
\end{equation*}
Now let $\gamma_l := 1 - c^l_{ab} $ for $a, b$ fixed. We note $s=\frac{2 \sqrt{2}}{3 \pi}$, from the series expansion we have that
$    \gamma_{l+1} = \gamma_{l} - s \gamma^{3/2}_{l} + O(\gamma^{5/2}_{l} )$ so that
\begin{align*}
    \gamma^{-1/2}_{l+1} &= \gamma^{-1/2}_{l} (1 - s \gamma^{1/2}_{l} + O(\gamma^{3/2}_{l} ) )^{-1/2}  = \gamma^{-1/2}_{l} (1 + \frac{s}{2} \gamma^{1/2}_{l} + O(\gamma^{3/2}_{l} ))\\
                        &= \gamma^{-1/2}_{l} + \frac{s}{2} + O(\gamma_{l}).
\end{align*}
Thus, as $l$ goes to infinity
\begin{equation*}
    \gamma^{-1/2}_{l+1} - \gamma^{-1/2}_l \sim \frac{s}{2}
\end{equation*}
and by summing and equivalence of positive divergent series
\begin{equation*}
    \gamma^{-1/2}_{l} \sim \frac{s}{2} l.
\end{equation*}
Therefore, we have $1 - c^l_{ab} \sim \frac{9\pi^2}{2 l^2}$. Using the same argument for $\overline{c}^l_{al}$, we conclude.

\end{proof}

\begin{prop2}\label{prop:existence_eoc}
Let $\phi \in \mathcal{D}^1_g$ be non ReLU-like function. Assume $V[\phi]$ is non-decreasing and $V[\phi']$ is non-increasing. Let $\sigma_{max}:=\sqrt{\sup_{x\geq 0}|x - \frac{V[\phi](x)}{V[\phi'](x)}|}$ and for $\sigma_b < \sigma_{max}$ let $q_{\sigma_b}$ be the smallest fixed point of the function $\sigma_b^2 + \frac{V[\phi]}{V[\phi']}$. Then we have $EOC = \{ (\sigma_b,\frac{1}{\sqrt{\mathbb{E}[\phi'(\sqrt{q}Z)^2]}} ): \sigma_b < \sigma_{max}\}$.
\end{prop2}

To prove Proposition \ref{prop:existence_eoc}, we need to introduce some lemmas. The next lemma gives a characterization of ReLU-like activation functions.

\begin{customlemma}{1.1}[A Characterization of ReLU-like activations]\label{prop:characterization}
Let $\phi \in \mathcal{D}^1(\mathbb{R}, \mathbb{R})$ such that $\phi(0)=0$ and $\phi'$ non-identically zero. We define the function $e$ for non-negative real numbers by 
\begin{equation*}
    e(x) =\frac{V[\phi](x)}{V[\phi'](x)} = \frac{\mathbb{E}[\phi(\sqrt{x}Z)^2]}{\mathbb{E}[\phi'(\sqrt{x}Z)^2]}
\end{equation*}
Then, for all $x\geq0$, $e(x)\leq x$.\\
Moreover, the following statements are equivalent
\begin{itemize}
    \item There exists $x_0>0$ such that $e(x_0) = x_0$.
    \item $\phi$ is ReLU-like, i.e. there exists $\lambda, \beta \in \mathbb{R}$ such that $\phi(x) =
\lambda x $ if $x>0$ and
$\phi(x) =\beta x $ if $x \leq 0$.
\end{itemize}
\end{customlemma}

\begin{proof}
Let $x>0$. We have for all $z \in \mathbb{R}, \phi(\sqrt{x}z) = \sqrt{x} \int_0^z \phi'(\sqrt{x}u) du$. This yields
\begin{align*}
    \mathbb{E}[\phi(\sqrt{x}Z)^2] &= x \mathbb{E}[\big(\int_0^Z \phi'(\sqrt{x}u) du\big)^2]\\
    &\leq x \mathbb{E}[ |Z| \int_0^{|Z|} \phi'(\sqrt{x}u)^2 du]\\
    &= x \mathbb{E}[ Z \int_0^Z \phi'(\sqrt{x}u)^2 du]\\
    &= x \mathbb{E}[ \phi'(\sqrt{x}Z)^2 du] 
\end{align*}
where we have used Cauchy-Schwartz inequality and Gaussian integration by parts. Therefore $e(x)\leq x$.\\
Now assume there exists $x_0>0$ such that $e(x_0) = x_0$. We have
\begin{align*}
    \mathbb{E}[\phi(\sqrt{x_0}Z)^2] &= x_0 \mathbb{E}[\big(\int_0^Z \phi'(\sqrt{x_0}u) du\big)^2]\\
    &= x_0 \mathbb{E}[1_{Z>0}\big(\int_0^Z \phi'(\sqrt{x_0}u) du\big)^2] + x_0 \mathbb{E}[1_{Z\leq0}\big(\int_Z^0 \phi'(\sqrt{x_0}u) du\big)^2] \\
    &\leq x_0 \mathbb{E}[ 1_{Z>0}  \int_0^{Z} 1 du \int_0^{Z} \phi'(\sqrt{x_0}u)^2 du] + x_0 \mathbb{E}[ 1_{Z\leq 0} \int_Z^{0} 1 du \int_Z^{0} \phi'(\sqrt{x_0}u)^2 du].\\
\end{align*}
The equality in Cauchy-Schwartz inequality implies that \\
- For almost every $z>0$, there exists $\lambda_z$ such that  $\phi'(\sqrt{x_0} u) = \lambda_z $ for all $u \in [0,z]$.\\
- For almost every $z<0$, there exists $\beta_z$ such that  $\phi'(\sqrt{x_0} u) = \beta_z $ for all $u \in [z,0]$.\\
Therefore, $\lambda_z, \beta_z$ are independent of $z$, and $\phi$ is ReLU-like.\\

It is easy to see that for ReLU-like activations, $e(x)=x$ for all $x\geq0$. 
\end{proof}

The next trivial lemma provides a sufficient condition for the existence of a fixed point of a shifted function.

\begin{customlemma}{1.2}
Let $g \in \mathcal{C}^0(\mathbb{R}^+, \mathbb{R})$ such that $g(0)=0$ and $g(x)\leq x$ for all $x \in \mathbb{R}^+$. Let $t_{max} := \sup_{x\geq 0}|x - g(x)|$ ($t_{max}$ may be infinite).
Then, for all $t \in [0,t_{max})$, the shifted function $t + g(.)$ has a fixed point.
\end{customlemma}

\begin{proof}
Let $t \in [0,t_{max})$. There exists $x_0 >0$ such that $t + g(.) < x_0 - g(x_0) + g(.)$. So we have $t + g(0)= t$ and $t + g(x_0) < x_0$, which means that $t + g(.)$ crosses the identity line, therefore the fixed point exists. 
\end{proof}

\begin{customcor}{1.1}\label{coro:general_eoc}
Let $\phi \in \mathcal{D}^1(\mathbb{R}, \mathbb{R})$ such that $\phi$ is non ReLU-like. Let $t_{max} = \sup_{x \geq 0} |x - \frac{V[\phi](x)}{V[\phi'](x)}|$. Then, For any $\sigma_b^2 \in [0,t_{max}) $, the shifted function $\sigma_b^2 + \frac{V[\phi]}{V[\phi']}$ has a fixed point q. Moreover, by taking $q$ to be the greatest fixed point, we have $\lim_{\sigma_b \rightarrow 0} q = 0$.
\end{customcor}

The limit of $q$ is zero because it is a fixed point of the function $\frac{V[\phi](x)}{V[\phi'](x)}$ which has only 0 as a fixed point for non ReLU-like functions. \\

Corollary \ref{coro:general_eoc} proves the existence of a fixed point for the shifted function $\sigma_b^2 + \frac{V[\phi]}{V[\phi']}$, which is a necessary condition for $(\sigma_b, 1/\sqrt{V[\phi'](q)})$ to be in the EOC where $q$ is the smallest fixed point. It is not a sufficient condition because $q$ may not be the smallest fixed point of $\sigma_b^2 + \frac{1}{V[\phi'](q)} V[\phi]$. We further analyse this problem hereafter.

\begin{definition}[Permissible couples]
Let $g ,h \in \mathcal{C}(\mathbb{R}^+, \mathbb{R}^+)$ and $c>0$. Define the function $k(x) = c +\frac{g(x)}{h(x)}$ for $x\geq0$ and let $q = \inf\{x : k(x) =x\}$. 
We say that $(g,h)$ is permissible if for any $c\geq0$ such that $q<\infty$, $q$ is the smallest fixed point of the function $c + \frac{g(.)}{h(q)}$.
\end{definition}

\begin{customlemma}{1.3}
Let $g ,h \in \mathcal{C}(\mathbb{R}^+, \mathbb{R}^+)$. Then the following statements are equivalent
\begin{enumerate}
    \item $(g,h)$ is permissible.
    \item For any $c>0$ such that $q$ is finite, we have $g(q) - g(x) < (q - x) h(q)$ for $x \in [0, q)$.
\end{enumerate}
\end{customlemma}

\begin{proof}
If $q$ is a fixed point of $c + \frac{g(.)}{h(.)}$, then $q$ is clearly a fixed point of $I(x) = c + \frac{1}{h(q)} g(x)$. Having $q$ is the smallest fixed point of $c + \frac{g(.)}{h(q)}$ is equivalent to $c + \frac{g(x)}{h(q)} > x$ for all $x \in [0,q)$. Since $c = q - \frac{g(q)}{h(q)}$, we conclude. 
\end{proof}

\begin{customcor}{1.2}
Let $g ,h \in \mathcal{C}(\mathbb{R}^+, \mathbb{R}^+)$. Assume $h$ is non-increasing, then $(g,h)$ is permissible.
\end{customcor}

\begin{proof}
Since $h$ is non-increasing, we have for $x \in [0, q)$, $g(q) - g(x) \leq h(q) 
(q - c)  - \frac{h(q)}{h(x)} g(x) = h(q) (q - (c+\frac{g(x)}{h(x)}))$. We conclude using the fact that $c+\frac{g(x)}{h(x)}>x$ for $x \in [0,q)$.
\end{proof}

\begin{customcor}{1.3}\label{cor:existence_eoc}
Let $\phi$ be a non ReLU-like function. Assume $V[\phi]$ is non-decreasing and $(V[\phi], V[\phi'])$ is permissible. Then, for any $\sigma_b^2 < t_{max} := \sup_{x\geq 0}|x - e(x)|$, by taking $\sigma_w^2 = \frac{1}{\mathbb{E}[\phi'(\sqrt{q}Z)^2]}$, we have $(\sigma_b, \sigma_w) \in EOC$. Moreover, we have $\lim_{\sigma_b \rightarrow 0} q = 0$.
\end{customcor}

We can omit the condition '$V[\phi]$ is non-decreasing' by choosing a small $t_{max}$. Indeed, by taking a small $\sigma_b$, the limiting variance $q$ is small, and we know that $V[\phi]$ is increasing near 0 because $V[\phi]'(0) = \phi'(0)^2>0$.\\

The proof of Proposition \ref{prop:existence_eoc} is straightforward from corollary A.3.\\

\begin{lemma2}
Let $\phi$ be a Tanh-like activation function, then $\phi$ satisfies all conditions of Proposition \ref{prop:existence_eoc} and $EOC = \{ (\sigma_b,\frac{1}{\sqrt{\mathbb{E}[\phi'(\sqrt{q}Z)^2]}} ): \sigma_b \in \mathbb{R}^+ \}$.
\end{lemma2}
\begin{proof}
For $x\geq0$, we have $V[\phi]'(x) = \frac{1}{x}\mathbb{E}[\sqrt{x}Z\phi'(\sqrt{x}Z)\phi(\sqrt{x}Z)] \geq0$, so $V[\phi]$ is non-decreasing. Moreover, $V[\phi']'(x) = \frac{1}{x}\mathbb{E}[\sqrt{x}Z\phi''(\sqrt{x}Z)\phi'(\sqrt{x}Z)] \leq0$, therefore $V[\phi']$ is non-increasing. To conclude, we still have to show that $t_{max} = \infty$.

Using the second condition on $\phi$, there exists $M>0$ such that $|\phi'(y)|^2 \geq M e^{-2\alpha |y|}$. Let $x>0$. we have 
\begin{align*}
    \mathbb{E}[\phi'(\sqrt{x}Z)^2] &\geq M \mathbb{E}[e^{-2\alpha |\sqrt{x}Z|}]\\
    &= 2 M \int_{0}^{\infty} e^{-2\alpha \sqrt{x}Z} \frac{e^{-z^2/2}}{\sqrt{2 \pi}} dz\\
    &= 2 M e^{2 \alpha^2 x} \Psi(2 \alpha \sqrt{x})\\
    &\sim \frac{2M}{2 \alpha \sqrt{x}}
\end{align*}
where $\Psi$ is the Gaussian cumulative function and where we used the asymptotic approximation $\Psi(x) \sim \frac{e^{-x^2/2}}{x}$ for large $x$.\\
Using this lower bound and the upper bound on $\phi$, there exists $x_0, k>0$ such that for $x>x_0$, we have $x - \frac{V[\phi](x)}{V[\phi'](x)} \geq x - k \sqrt{x} \rightarrow \infty$ which concludes the proof.
\end{proof}

\begin{prop2}[Convergence rate for smooth activations]
\label{prop:rate_smooth_functions}
Let $\phi \in \mathcal{A}$ such that $\phi$ non-linear (i.e. $\phi^{(2)}$ is non-identically zero). Then, on the EOC, we have $1 - c^l \sim \frac{\beta_q}{l}$ where $\beta_q = \frac{2 \mathbb{E}[\phi'(\sqrt{q}Z)^2]}{q \mathbb{E}[\phi''(\sqrt{q} Z)^2]}$.
\end{prop2}

\begin{proof}
We first prove that $\lim_{l \rightarrow \infty} c^l = 1$ on the EOC. Let $x \in [0,1)$ and $u_2(x):=x Z_1 + \sqrt{1-x^2}Z_2$, we have
\begin{align*}
    f'(x) &= \sigma^2_w \mathbb{E}[\phi'(\sqrt{q}Z_1)\phi'(\sqrt{q}u_2(x))]\\
    &\leq  \sigma^2_w  (\mathbb{E}[\phi'(\sqrt{q}Z_1)^2])^{1/2} (\mathbb{E}[\phi'(\sqrt{q}u_2(x))^2])^{1/2}\\
    &= 1
\end{align*}
where we have used Cauchy Schwartz inequality and the fact the $\sigma_w^2 = \frac{1}{\mathbb{E}[\phi'(\sqrt{q}Z)^2]}$. Moreover, the equality holds if and only if there exists a constant $s$ such that $\phi'(\sqrt{q}(x z_1 + \sqrt{1-x^2}z_2)) = s \phi'(\sqrt{q}z_1)$ for almost any $z_1, z_2 \in \mathbb{R}$, which is equivalent to having $\phi'$ equal to a constant almost everywhere on $\mathbb{R}$, hence $\phi$ is linear and $q$ does not exists. This proves that for all $x \in [0,1)$, $f'(x)<1$. Integrating both sides between $x$ and 1 yields $f(x)>x$ for all $x \in [0,1)$. Therefore $c^l$ is non-decreasing and converges to the fixed point of $f$ which is 1. \\

Now we want to prove that $f$ admits a Taylor expansion near 1. It is easy to do that if $\phi \in \mathcal{D}^3_{g}$. Indeed, using the conditions on $\phi$, we can easily see that $f$ has a third derivative at 1 and we have 
\begin{align*}
f'(1) &= \sigma_w^2 \mathbb{E}[\phi'(\sqrt{q}Z)^2]\\
f''(1) &= \sigma_w^2 q \mathbb{E}[\phi''(\sqrt{q}Z)^2].
\end{align*}
A Taylor expansion near 1 yields 
\begin{align*}
f(x) &= 1 + f'(1)(x-1) + \frac{(x-1)^2}{2} f''(1) + O((x-1)^3)\\
&= x + \frac{(x-1)^2}{\beta_q} + O((x-1)^3).
\end{align*}
The proof is a bit more complicated for general $\phi \in \mathcal{A}$. We prove the result when $\phi^{(2)}(x) = 1_{x<0} g_1(x) + 1_{x\geq 0} g_2(x)$. The generalization to the whole class is straightforward. Let us first show that there exists $g \in \mathcal{C}^1$ such that $f^{(3)}(x) = \frac{1}{\sqrt{1 - x^2}} g(x)$. \\
We have
\begin{align*}
    f''(x) &= \sigma_w^2 q \mathbb{E}[\phi''(\sqrt{q}Z_1) \phi''(\sqrt{q}U_2(x))]\\
    &= \sigma_w^2 q \mathbb{E}[\phi''(\sqrt{q}Z_1) 1_{U_2(x) <0}g_1(\sqrt{q}U_2(x))] + \sigma_w^2 q \mathbb{E}[\phi''(\sqrt{q}Z_1) 1_{U_2(x) >0}g_2(\sqrt{q}U_2(x))].
\end{align*}
Let $G(x) = \mathbb{E}[\phi''(\sqrt{q}Z_1) 1_{U_2(x) <0}g_1(\sqrt{q}U_2(x))]$ then
\begin{align*}
    G'(x) &=  \mathbb{E}[\phi''(\sqrt{q}Z_1) (Z_1 - \frac{x}{\sqrt{1 - x^2}}Z_2)\delta_{U_2(x) =0} \frac{1}{\sqrt{1 - x^2}} g_1(\sqrt{q}U_2(x))] \\
    &+ \mathbb{E}[\phi''(\sqrt{q}Z_1) 1_{U_2(x) <0} \sqrt{q}(Z_1 - \frac{x }{\sqrt{1 - x^2}} Z_2)g'_1(\sqrt{q}U_2(x))].
\end{align*}

After simplification, it is easy to see that $G'(x) = \frac{1}{\sqrt{1 - x^2}} G_1(x)$ where $G_1 \in \mathcal{C}^1$. By extending the same analysis to the second term of $f''$, we conclude that there exists $g \in \mathcal{C}^1$ such that $f^{(3)}(x) = \frac{1}{\sqrt{1 - x^2}} g(x)$. \\

Let us now derive a Taylor expansion of $f$ near 1. Since $f^{(3)}$ is potentially non defined at 1, we use the change of variable $x = 1 - t^2$ to compensate this effect. Simple algebra shows that the function $t \rightarrow f(1 - t^2)$ has a Taylor expansion near 0

\begin{align*}
    f(1 - t^2) = 1 - t^2 f'(1) + \frac{t^4}{2} f''(1) + O(t^5).
\end{align*}
Therefore,
\begin{equation*}
    f(x) = 1 + (x-1) f'(1) + \frac{(x-1)^2}{2} f''(1) + O((x-1)^{5/2}).
\end{equation*}
Note that this expansion is weaker than the expansion when $\phi \in \mathcal{D}^3_g$.

Denote $\lambda_l := 1 - c^l$, we have 
\begin{align*}
\lambda_{l+1} &= \lambda_l - \frac{\lambda_l^2}{\beta_q} + O(\lambda_l^{5/2}) 
\end{align*}
therefore,
\begin{align*}
\lambda_{l+1}^{-1} &= \lambda_l^{-1}(1 - \frac{\lambda_l}{\beta_q} + O(\lambda_l^{3/2}))^{-1} \\
&= \lambda_l^{-1}(1 + \frac{\lambda_l}{\beta_q} + O(\lambda_l^{3/2})) \\
&= \lambda_l^{-1} + \frac{1}{\beta_q} + O(\lambda_l^{1/2}). \\
\end{align*}
By summing (divergent series), we conclude that $\lambda_l^{-1} \sim \frac{l}{\beta_q}$.
\end{proof}
\medskip
\medskip

\begin{prop2}\label{prop:eoc_approximation}
Let $\phi \in \mathcal{A}$ be a non-linear activation function such that $\phi(0)=0$, $\phi'(0)\neq0$. Assume that $V[\phi]$ is non-decreasing  and $V[\phi'])$ is non-increasing, and let $\sigma_{max}>0$ be defined as in Proposition \ref{prop:existence_eoc}. Define the gradient with respect to the $l^{th}$ layer by $\frac{\partial E}{\partial y^l} = (\frac{\partial E}{\partial y^l_i})_{1\leq i\leq N_l}$ and let $\Tilde{Q}^l_{ab} = \mathbb{E}[\frac{\partial E}{\partial y^l_a} ^{T} \frac{\partial E}{\partial y^l_b}]$ denote the covariance matrix of the gradients during backpropagation. Recall that $\beta_q = \frac{2 \mathbb{E}[\phi'(\sqrt{q}Z)^2]}{q\mathbb{E}[\phi''(\sqrt{q} Z)^2]}$.

Then, for any $\sigma_b<\sigma_{max}$, by taking $(\sigma_b, \sigma_w) \in EOC$ we have
\begin{itemize}
    \item $\sup_{x \in [0,1]} |f(x) - x| \leq \frac{1}{\beta_q} $
    \item For $l \geq 1$, $\huge| \frac{\Tr(\Tilde{Q}^l_{ab})}{\Tr(\Tilde{Q}^{l+1}_{ab})} - 1\huge| \leq \frac{2}{\beta_q} $
\end{itemize}
Moreover, we have
\begin{equation*}
    \lim_{\substack{\sigma_b \rightarrow 0 \\ (\sigma_b, \sigma_w) \in \textrm{EOC}}} \beta_q = \infty.
\end{equation*}
\end{prop2}

To prove this result, let us first prove a more general result.

\begin{prop2}[How close is $f$ to the identity function?]
\label{prop:distance_from_identity}
Let $\phi \in \mathcal{D}^2(\mathbb{R}, \mathbb{R})-\{0\}$ and $(\sigma_b, \sigma_w) \in D_{\phi, var}$ with $q$ the corresponding limiting variance. Then,
\begin{equation*}
    \sup_{x \in [0,1]}|f(x) - x| \leq |\sigma_w^2 \mathbb{E}[\phi'(\sqrt{q}Z)^2] - 1| +  \frac{\sigma_w^2}{2} q\mathbb{E}[\phi''(\sqrt{q}Z)^2]
\end{equation*}
\end{prop2}

\begin{proof}
Using a second order Taylor expansion, we have for all $s \in [0,1]$
\begin{equation*}
    |f(x) - f(1) - f'(1)(x-1)| \leq \frac{(1-x)^2}{2} \sup_{\theta \in [0,1]} |f''(\theta)|.
\end{equation*}
We have $f(1)=1$. Therefore $|f(x) - x| \leq (1-x)  |f'(1) - 1|  + \frac{(1-x)^2}{2} \sup_{\theta \in [0,1]} |f''(\theta)|$. \\
For $\theta \in [0,1]$, we have 
\begin{align*}
    f''(\theta) &= \sigma_w^2 q \mathbb{E}[\phi''(\sqrt{q}Z_1) \phi''(\sqrt{q}U_2(\theta))]\\
                &\leq \sigma_w^2 q \mathbb{E}[\phi''(\sqrt{q}Z)^2]\\
                &= \frac{\sigma_w^2}{2} q \mathbb{E}[\phi''(\sqrt{q}Z)^2]
\end{align*}
using Cauchy-Schwartz inequality.
\end{proof}

As a result, for $\phi \in \mathcal{D}^2(\mathbb{R}, \mathbb{R})-\{0\}$ and $(\sigma_b, \sigma_w) \in EOC$ with $q$ the corresponding limiting variance, we have
\begin{equation*}
    \sup_{x \in [0,1]}|f(x) - x| \leq \frac{q\mathbb{E}[\phi''(\sqrt{q}Z)^2]}{2\mathbb{E}[\phi'(\sqrt{q}Z)^2]} = \frac{1}{\beta_q}
\end{equation*}
which is the first result of Proposition \ref{prop:eoc_approximation}. \\

Now let us prove the second result for gradient backpropagation, we show that under some assumptions, our results of forward information propagation generalize to the back-propagation of the gradients. Let us first recall the results in \cite{samuel} (we use similar notations hereafter).\\

Let $E$ be the loss we want to optimize. The backpropagation process is given by the equations
\begin{align*}
    \frac{\partial E}{\partial W^l_{ij}} &= \delta^l_i \phi(y_j^{l-1})\\
    \delta^l_i = \frac{\partial E}{\partial y^l_i} &= \phi'(y^l_i) \sum_{j=1}^{N_{l+1}} \delta^{l+1}_j W^{l+1}_{ji}.
\end{align*}
Although $\delta^l_i$ is non Gaussian (unlike $y^l_i$), knowing how $\Tilde{q}^l_{a} = \mathbb{E}[(\delta^l_i)^2]$ changes back through the network will give us an idea about how the norm of the gradient changes. Indeed, following this approach, and using the approximation that the weights used during forward propagation are independent from those used for backpropagation, \cite{samuel} showed that 
\begin{equation*}
    \Tilde{q}^l_a = \Tilde{q}^{l+1}_a \frac{N_{l+1}}{N_l} \chi_1
\end{equation*}
where $\chi_1 = \sigma_w^2 \mathbb{E}[\phi'(\sqrt{q}Z)^2]$. \\

Considering a constant width network, authors concluded that $\chi_1$ controls also the depth scales of the gradient norm, i.e. $\Tilde{q}^l_a = \Tilde{q}^L_a e^{-(L-l/\xi_{\Delta})}$ where $\xi^{-1}_{\Delta} = - \log(\chi_1)$. So in the ordered phase, gradients can propagate to a depth of $\xi_{\Delta}$ without being exponentially small, while in the chaotic phase, gradient explode exponentially. On the EOC ($\chi_1 = 1$), the depth scale is infinite so the gradient information can also propagate deeper without being exponentially small.\\
The following result shows that our previous analysis on the EOC extends to the backpropagation of gradients, and that we can make this propagation better by choosing a suitable activation function and an initialization on the EOC. We use the following approximation to ease the calculations: the weights used in forward propagation are independent from those used in backward propagation.

\begin{prop2}[Better propagation for the gradient]
Let $a$ and $b$ be two inputs and $(\sigma_b, \sigma_w) \in D_{\phi, var}$ with $q$ the limiting variance. We define the covariance between the gradients with respect to layer $l$ by $\Tilde{q}^l_{ab} = \mathbb{E}[\delta^l_i(a) \delta^l_i(b)]$. Then, we have 
\begin{equation*}
    \huge|\frac{\Tilde{q}^l_{ab}}{\Tilde{q}^{l+1}_{ab}} \times \frac{N_{l}}{N_{l+1}} - 1 \huge| \leq |\sigma_w^2 \mathbb{E}[\phi'(\sqrt{q}Z)^2] - 1| + (1 - c^l_{ab}) \sigma_w^2 q \mathbb{E}[\phi''(\sqrt{q}Z)^2] \rightarrow_{\sigma_b \rightarrow 0} 0.
\end{equation*}
\end{prop2}

\begin{proof}
We have 
\begin{align*}
    \Tilde{q}^l_{ab} &= \mathbb{E}[\delta^l_i(a) \delta^l_i(b)]\\
    &= \mathbb{E}[\phi'(y^l_i(a)) \phi'(y^l_i(b)) \sum_{j=1}^{N_{l+1}} \delta^{l+1}_j(a) W^{l+1}_{ji} \sum_{j=1}^{N_{l+1}} \delta^{l+1}_j(b) W^{l+1}_{ji}]\\
    &= \mathbb{E}[\phi'(y^l_i(a)) \phi'(y^l_i(b))] \times \mathbb{E}[\delta^{l+1}_j(a)\delta^{l+1}_j(b)] \times \mathbb{E} [\sum_{j=1}^{N_{l+1}} (W^{l+1}_{ji})^2]\\
    &\approx \Tilde{q}^{l+1}_{ab} \frac{N_{l+1}}{N_{l}} \sigma_w^2 \mathbb{E}[\phi'(\sqrt{q}Z_1)\phi'(\sqrt{q}U_2(c^l_{ab}))]\\
    &=\Tilde{q}^{l+1}_{ab} \frac{N_{l+1}}{N_{l}} f'(c^l_{ab}).
\end{align*}
We conclude using the fact that $|f'(x) - 1| \leq |f'(1)-1| + (1-x) f''(1)$
\end{proof}
The dependence in the width of the layer is natural since it acts as a scale for the covariance. We define the gradient with respect to the $l^{th}$ layer by $\frac{\partial E}{\partial y^l} = (\frac{\partial E}{\partial y^l_i})_{1\leq i\leq N_l}$ and let $\Tilde{Q}^l_{ab} = \mathbb{E}[\frac{\partial E}{\partial y^l_a} ^{T} \frac{\partial E}{\partial y^l_b}]$ denote the covariance matrix of the gradients during backpropagation. Then, on the EOC, we have 
\begin{equation*}
    \huge| \frac{\Tr(\Tilde{Q}^l_{ab})}{\Tr(\Tilde{Q}^{l+1}_{ab})} - 1\huge| \leq (1 - c^l_{ab}) \frac{q \mathbb{E}[\phi''(\sqrt{q}Z)^2]}{ \mathbb{E}[\phi'(\sqrt{q}Z)^2]} \leq \frac{2}{\beta_q}.
\end{equation*}

So again, the quantity $|\phi|_{EOC}$ controls the vanishing of the covariance of the gradients during backpropagation. This was expected because linear activation functions do not change the covariance of the gradients.

\section{Further theoretical results}

\subsection{Results on the Edge of Chaos}

The next lemma shows that under some conditions, the EOC does not include couples $(\sigma_b, \sigma_w)$ with small $\sigma_b>0$. 

\begin{lemma2}[Trivial EOC]
\label{lemma:trivial_eoc}
Assume there exists $M>0$ such that $\mathbb{E}[\phi''(xZ) \phi(xZ)] > 0$ for all $x \in ]0,M[$. Then, there exists $\sigma>0$ such that $EOC \cap ([0,\sigma) \times \mathbb{R}^+) = \{ (0 , \frac{1}{|\phi'(0)|} )\}$. Moreover, if $M =\infty$ then $EOC = \{ (0 , \frac{1}{|\phi'(0)|} )\}$.
\end{lemma2}
Activation functions that satisfy the conditions of Lemma \ref{lemma:trivial_eoc} cannot be used with small $\sigma_b >0$ (note that using $\sigma_b = 0$ would lead to $q=0$ which is not practical for the training), therefore, the result of Proposition \ref{prop:eoc_approximation} do not apply in this case. However, as we will see hereafter, SiLU (a.k.a Swish) has a partial EOC, and still allows better information propagation (Proposition \ref{prop:rate_smooth_functions}) compared to ReLU even if $\sigma_b$ not very small. 

\begin{proof}
It is clear that $(0 , \frac{1}{|\phi'(0)|} ) \in EOC$. For $\sigma_b>0$ we denote by $q$ the smallest fixed point of the function $\sigma_b^2 + \frac{V[\phi]}{V[\phi']}$ (which is supposed to be the limiting variance on the EOC). Using the condition on $\phi$ and the fact that $\lim_{\sigma_b \rightarrow 0} q = 0$, there exists $\sigma>0$ such that for $\sigma_b < \sigma$ we have $\mathbb{E}[\phi''(\sqrt{q}Z)\phi(\sqrt{q}Z)] >0$. Now let us prove that for $\sigma_b \in ]0,\sigma[$, the limiting variance does not satisfy the EOC equation. \\
Let $t_{max} = \sqrt{\sup_{x>0}|x - \frac{V[\phi]}{V[\phi']}|}$ and $\sigma_b \in ]0, \min(t_{max}, \sigma) [$. Recall that for all $x\geq 0$ we have that
\begin{equation*}
    F'(x) = \sigma_w^2 (\mathbb{E}[\phi'(\sqrt{x}Z)^2] + \mathbb{E}[\phi''(\sqrt{x}Z)\phi(\sqrt{x}Z)])
\end{equation*}
 Using $\sigma_w^2 = 1/V[\phi'](q)$ (EOC equation) we have that $F'(q) = 1 + \sigma_w^2 \mathbb{E}[\phi''(\sqrt{q}Z)\phi(\sqrt{q}Z)]) > 1$. Therefore, the function $\sigma_b^2 + \frac{1}{V[\phi'](q)} V[\phi]$ crosses the identity in a point $\hat{q} < q$, hence $(\sigma_b, \sigma_w) \not\in D_{\phi, var}$. Therefore, for any $\sigma_b \in ]0,\sigma[$, there is no $\sigma_w$ such that $(\sigma_b, \sigma_w) \in EOC$. \\
 
 If $M = \infty$, the previous analysis is true for any $\sigma>0$, by taking the limit $\sigma \rightarrow \infty$, we conclude. 
\end{proof}

This is true for activations such as Shifted Softplus (a shifted version of Softplus in order to have $\phi(0) = 0$) and SiLU (a.k.a Swish).
\begin{corollary}\label{cor:functions_without_eoc}
$EOC_{SSoftplus} = \{ (0 , 2)\}$ and there exists $\sigma>0$ such that $EOC_{SiLU} \cap ([0,\sigma[ \times \mathbb{R}^+) = \{ (0 , 2 )\}$
\end{corollary}

\begin{proof}
let $s(x)=\frac{1}{1+e^{-x}}$ for all $x \in \mathbb{R}$ (sigmoid function).
\begin{enumerate}
    \item Let $sp(x) = \log(1+e^x) - \log(2)$ for $x \in \mathbb{R}$ (Shifted Softplus). We have $sp'(x) = s(x)$ and $sp''(x) = s(x)(1 - s(x))$. For $x>0$ we have
    \begin{align*}
        \mathbb{E}[sp''(xZ)sp(xZ)] &= \mathbb{E}[s(xZ)(1 - s(xZ))sp(xZ)]\\
                                    &= \mathbb{E}[1_{Z>0}(s(xZ)(1 - s(xZ))sp(xZ))] + \mathbb{E}[1_{Z<0}(s(xZ)(1 - s(xZ))sp(xZ))]\\
                                    &= \mathbb{E}[1_{Z>0}(s(xZ)(1 - s(xZ))sp(xZ))] + \mathbb{E}[1_{Z<0}(s(xZ)(1 - s(xZ))sp(-xZ))]\\
                                    &= \mathbb{E}[1_{Z>0}(s(xZ)(1 - s(xZ))(sp(xZ)+sp(-xZ)))]>0,\\
    \end{align*}
    where we have used the fact that $sp(y) + sp(-y) = \log(\frac{2 + e^y + e^{-y}}{4}) > 0$ for all $y>0$. We conclude using Lemma \ref{lemma:trivial_eoc}.
    
    \item Let $si(x) = x s(x)$ (SiLU activation function, known also as Swish). We have $si'(x) = s(x) + xs(x)(1-s(x))$ and $si''(x) = s(x)(1-s(x))(2 + x(1 - 2s(x)))$. Using the same technique as for SSoftplus, we have for $x>0$
    \begin{align*}
        \mathbb{E}[si''(xZ)si(xZ)] &= \mathbb{E}[xZ \times s(xZ)^2\times(1-s(xZ))(2 + xZ(1-2))]\\
        &= \mathbb{E}[1_{Z>0} G(xZ)],\\
    \end{align*}
    where $G(y) = y s(y) ( 1 -s(y)) (2 + y(1-2s(y)))(2s(y)-1)$. The only term that changes sign is $(2 + y(1-2s(y)))$. It is positive for small $y$ and negative for large $y$. We conclude that there $M>0$ such that $\mathbb{E}[si''(xZ)si(xZ)]>0$ for $x \in ]0,M[$.
    
\end{enumerate}
\end{proof}

\subsection{Beyond the Edge of Chaos}
\label{subsec:beyond_eoc}

Can we make the distance between $f$ and the identity function small independently from the choice of $\sigma_b$? The answer is yes if we select the right activation function. Let us first define a semi-norm on $\mathcal{D}^2(\mathbb{R}, \mathbb{R})$.

\begin{definition}[EOC semi-norm]
The semi-norm $|.|_{EOC}$ is defined on $\mathcal{D}^2(\mathbb{R}, \mathbb{R})$ by $|\phi|_{EOC} = \sup_{y \in \mathbb{R}^+} \frac{y\mathbb{E}[\phi''(\sqrt{y}Z)^2]}{\mathbb{E}[\phi'(\sqrt{y}Z)^2]}$.\\
$|.|_{EOC}$ is a norm on the quotient space $\mathcal{D}^2(\mathbb{R}, \mathbb{R}) / \mathcal{L}(\mathbb{R})$ where $\mathcal{L}(\mathbb{R})$ is the space of linear functions.
\end{definition}

When $|\phi|_{EOC}$ is small, $\phi$ is close to a linear function, which implies that the function $\frac{V[\phi]}{V[\phi']}$ defined on $\mathbb{R}^+$ is close to the identity function. Thus, for a fixed $\sigma_b$, we expect $q$ to become arbitrarily big when $|\phi|_{EOC}$ goes to zero.
\begin{customlemma}{2.1}\label{lemma:limit_q_phi_linear}
Let $(\phi_{n})_{n \in \mathbb{N}}$ be a sequence of functions such that $\lim_{n \rightarrow \infty} |\phi_n|_{EOC} = 0$. Let $\sigma_b>0$ and assume that for all $n \in \mathbb{N}$ there exists $\sigma_{w,n}$ such that $(\sigma_w,\sigma_{w,n}) \in EOC$. Let $q_n$ be the limiting variance. Then $\lim_{n \rightarrow \infty} q_n = \infty$
\end{customlemma}
\begin{proof}
The proof is straightforward knowing that $f(0) \leq \frac{1}{2} |\phi_n|_{EOC}$, which implies that $\frac{\sigma_b^2}{q} \leq \frac{1}{2} |\phi_n|_{EOC}$.
\end{proof}
\begin{customcor}{2.1}
\label{cor:distance_from_identity_eoc_for_any_sigmab}
Let $\phi \in \mathcal{D}^2(\mathbb{R}, \mathbb{R})-\{0\}$ and $(\sigma_b, \sigma_w) \in EOC$ with $q$ the corresponding limiting variance. Then,
\begin{equation*}
    \sup_{x \in [0,1]}|f(x) - x| \leq \frac{1}{2} |\phi|_{EOC}.
\end{equation*}
\end{customcor}

Corollary 2.1 shows that by taking an activation function $\phi$ such that $|\phi|_{EOC}$ is small and by initializing the network on the EOC, the correlation function is close to the identity function, i.e., the signal propagates deeper through the network. However, note that there is a trade-off to take in account here: we loose expressiveness by taking $|\phi|_{EOC}$ too small, because this would imply that $\phi$ is close to a linear function. So there is a trade-off between signal propagation and expressiveness We check this finding with activation functions of the form $\phi_{\alpha}(x) = x + \alpha \text{Tanh}(x)$. Indeed, we have $|\phi_{\alpha}|_{EOC} \leq \alpha^2 \sup_{y \in \mathbb{R}^+} \mathbb{E}[\text{Tanh}''(\sqrt{x} Z)^2] \rightarrow_{\alpha \rightarrow 0} 0$. So by taking small $\alpha$, we would theoretically provide deeper signal propagation. However, note that we loose expressiveness as $\alpha$ goes to zero because $\phi_{\alpha}$ becomes closer to the identity function. So There is also a trade-off here. The difference with Proposition \ref{prop:eoc_approximation} is that here we can compensate the expressiveness issue by adding more layers (see e.g. \cite{montufar} who showed that expressiveness grows exponentially with depth).\\

\section{Experiments}

\subsection{Training with RMSProp}
For RMSProp, the learning rate $10^{-5}$ is nearly optimal for networks with depth $L\leq200$ (for deeper networks, $10^{-6}$ gives better results). This learning rate was found by a grid search with exponential step of size 10. \\
Figure \ref{fig:rmsprop} shows the training curves of ELU, ReLU and Tanh on MNIST for a network with depth 200 and width 300. Here also, ELU and Tanh perform better than ReLU. This confirms that the result of Proposition \ref{prop:rate_smooth_functions} is independent of the training algorithm.
\\
ELU has faster convergence than Tanh. This could be explained by the saturation problem of Tanh.
\begin{figure}
\begin{center}
\includegraphics[scale=0.4]{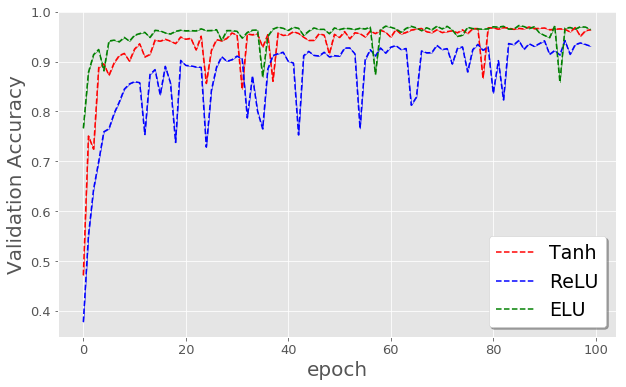}
\caption{100 epochs of the training curves of ELU, ReLU and Tanh networks of depth 200 and width 300 on MNIST with RMSProp}
\label{fig:rmsprop}
\end{center}
\end{figure}

\subsection{Training with activation $\phi_{\alpha}(x) = x + \alpha \text{Tanh}(x)$}
As we have already mentioned, $\phi_{\alpha}$ satisfies all conditions of Proposition \ref{prop:rate_smooth_functions}. Therefore, we expect it to perform at least better than ReLU for deep neural networks. Figure \ref{fig:phi_alpha} shows the training curve for width 300 and depth 200 with different activation functions. $\phi_{0.5}$ has approximately similar performance as ELU and better than Tanh and ReLU. Note that $\phi_{\alpha}$ does not suffer form saturation of the gradient, which could explain why it performs better than Tanh.

\subsection{Impact of $\phi''(0)$}
Since we usually take $\sigma_b$ small on the EOC, then having $\phi''(0)=0$ would make the coefficient $\beta_q$ even bigger. We test this result on SiLU (a.k.a Swish) for depth 70. SiLU is defined by 
$$\phi_{SiLU}(x) = x~\text{sigmoid}(x)$$
we have $\phi''(0) = 1/2$. consider a modified SiLU (MSiLU) defined by 
$$\phi_{MSiLU}(x) = x~\text{sigmoid}(x) + (e^{-x^2} - 1)/4$$
We have $\phi_{MSiLU}''(0)=0$.\\
Figure \ref{fig:swishplusexp} shows the the training curves (test accuracy) of SiLU and MSiLU on MNIST with SGD. MSiLU performs better than SiLU, expecially at the beginning of the training. 

\begin{figure}
\begin{center}
\includegraphics[scale=0.4]{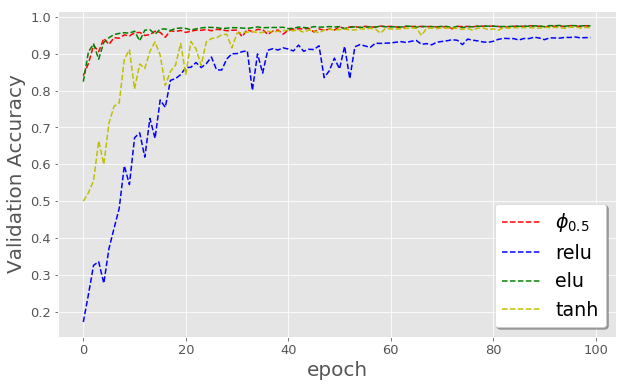}
\caption{100 epochs of the training curves of ELU, ReLU, Tanh and $\phi_{0.5}$ networks of depth 200 and width 300 on MNIST with SGD}
\label{fig:phi_alpha}
\end{center}
\end{figure}

\begin{figure}
\begin{center}
\includegraphics[scale=0.4]{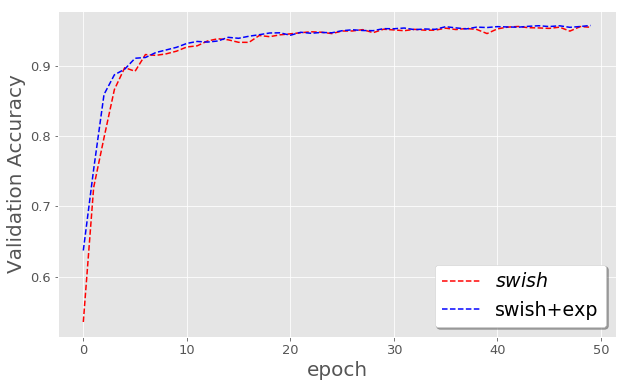}
\caption{50 epochs of the training curves of SiLU and MSiLU on MNIST with SGD}
\label{fig:swishplusexp}
\end{center}
\end{figure}

\end{document}